\documentclass[letterpaper,english]{article}
\usepackage[T1]{fontenc}
\usepackage[latin9]{inputenc}
\PassOptionsToPackage{natbib=true}{biblatex}
\usepackage{array}
\usepackage{mathtools}
\usepackage{amsmath}
\usepackage{amsthm}
\usepackage{amssymb}
\usepackage{graphicx}
\usepackage{xargs}[2008/03/08]

\makeatletter

\newcommand{\noun}[1]{\textsc{#1}}
\providecommand{\tabularnewline}{\\}

\theoremstyle{plain}
\newtheorem{thm}{\protect\theoremname}
\theoremstyle{plain}
\newtheorem{lem}[thm]{\protect\lemmaname}
\ifx\proof\undefined
\newenvironment{proof}[1][\protect\proofname]{\par
	\normalfont\topsep6\p@\@plus6\p@\relax
	\trivlist
	\itemindent\parindent
	\item[\hskip\labelsep\scshape #1]\ignorespaces
}{%
	\endtrivlist\@endpefalse
}
\providecommand{\proofname}{Proof}
\fi

\def\year{2021}\relax
\usepackage{aaai21}
\usepackage{times}
\usepackage[hyphens]{url}
\usepackage{caption}
\usepackage{amsmath}
\title{Improving Ensemble Robustness by Collaboratively \\ Promoting and Demoting Adversarial Robustness}
\author{Anh Bui\textsuperscript{\rm 1}, 
		Trung Le\textsuperscript{\rm 1}, 
		He Zhao\textsuperscript{\rm 1} \\
		Paul Montague\textsuperscript{\rm 2}, 
		Olivier deVel\textsuperscript{\rm 2}, 
		Tamas Abraham\textsuperscript{\rm 2}, 
		Dinh Phung\textsuperscript{\rm 1}}
\affiliations{

    \textsuperscript{\rm 1}Monash University \\
	\textsuperscript{\rm 2}Defence Science and Technology Group, Australia \\
	tuananh.bui@monash.edu

}

\@ifundefined{showcaptionsetup}{}{%
 \PassOptionsToPackage{caption=false}{subfig}}
\usepackage{subfig}
\makeatother

\usepackage{babel}
\usepackage[style=authoryear]{biblatex}
\providecommand{\lemmaname}{Lemma}
\providecommand{\theoremname}{Theorem}

\begin{document}
\maketitle

\global\long\def\sidenote#1{\marginpar{\small\emph{{\color{Medium}#1}}}}%

\global\long\def\se{\hat{\text{se}}}%
\global\long\def\interior{\text{int}}%
\global\long\def\boundary{\text{bd}}%
\global\long\def\ML{\textsf{ML}}%
\global\long\def\GML{\mathsf{GML}}%
\global\long\def\HMM{\mathsf{HMM}}%
\global\long\def\support{\text{supp}}%
\global\long\def\new{\text{*}}%
\global\long\def\stir{\text{Stirl}}%
\global\long\def\mA{\mathcal{A}}%
\global\long\def\mB{\mathcal{B}}%
\global\long\def\expect{\mathbb{E}}%
\global\long\def\mF{\mathcal{F}}%
\global\long\def\mK{\mathcal{K}}%
\global\long\def\mH{\mathcal{H}}%
\global\long\def\mX{\mathcal{X}}%
\global\long\def\mZ{\mathcal{Z}}%
\global\long\def\mS{\mathcal{S}}%
\global\long\def\Ical{\mathcal{I}}%
\global\long\def\mT{\mathcal{T}}%
\global\long\def\Pcal{\mathcal{P}}%
\global\long\def\dist{d}%
\global\long\def\HX{\entro\left(X\right)}%
\global\long\def\entropyX{\HX}%
\global\long\def\HY{\entro\left(Y\right)}%
\global\long\def\entropyY{\HY}%
\global\long\def\HXY{\entro\left(X,Y\right)}%
\global\long\def\entropyXY{\HXY}%
\global\long\def\mutualXY{\mutual\left(X;Y\right)}%
\global\long\def\mutinfoXY{\mutualXY}%
\global\long\def\given{\mid}%
\global\long\def\gv{\given}%
\global\long\def\goto{\rightarrow}%
\global\long\def\asgoto{\stackrel{a.s.}{\longrightarrow}}%
\global\long\def\pgoto{\stackrel{p}{\longrightarrow}}%
\global\long\def\dgoto{\stackrel{d}{\longrightarrow}}%
\global\long\def\lik{\mathcal{L}}%
\global\long\def\logll{\mathit{l}}%
\global\long\def\bigcdot{\raisebox{-0.5ex}{\scalebox{1.5}{\ensuremath{\cdot}}}}%
\global\long\def\sig{\textrm{sig}}%
\global\long\def\likelihood{\mathcal{L}}%
\global\long\def\vectorize#1{\mathbf{#1}}%

\global\long\def\vt#1{\mathbf{#1}}%
\global\long\def\gvt#1{\boldsymbol{#1}}%
\global\long\def\idp{\ \bot\negthickspace\negthickspace\bot\ }%
\global\long\def\cdp{\idp}%
\global\long\def\das{}%
\global\long\def\id{\mathbb{I}}%
\global\long\def\idarg#1#2{\id\left\{  #1,#2\right\}  }%
\global\long\def\iid{\stackrel{\text{iid}}{\sim}}%
\global\long\def\bzero{\vt 0}%
\global\long\def\bone{\mathbf{1}}%
\global\long\def\a{\mathrm{a}}%
\global\long\def\ba{\mathbf{a}}%
\global\long\def\b{\mathrm{b}}%
\global\long\def\bb{\mathbf{b}}%
\global\long\def\B{\mathrm{B}}%
\global\long\def\boldm{\boldsymbol{m}}%
\global\long\def\c{\mathrm{c}}%
\global\long\def\C{\mathrm{C}}%
\global\long\def\d{\mathrm{d}}%
\global\long\def\D{\mathrm{D}}%
\global\long\def\N{\mathrm{N}}%
\global\long\def\h{\mathrm{h}}%
\global\long\def\H{\mathrm{H}}%
\global\long\def\bH{\mathbf{H}}%
\global\long\def\K{\mathrm{K}}%
\global\long\def\M{\mathrm{M}}%
\global\long\def\bff{\vt f}%
\global\long\def\bx{\mathbf{\mathbf{x}}}%

\global\long\def\bl{\boldsymbol{l}}%
\global\long\def\s{\mathrm{s}}%
\global\long\def\T{\mathrm{T}}%
\global\long\def\bu{\mathbf{u}}%
\global\long\def\v{\mathrm{v}}%
\global\long\def\bv{\mathbf{v}}%
\global\long\def\bo{\boldsymbol{o}}%
\global\long\def\bh{\mathbf{h}}%
\global\long\def\bs{\boldsymbol{s}}%
\global\long\def\x{\mathrm{x}}%
\global\long\def\bx{\mathbf{x}}%
\global\long\def\bz{\mathbf{z}}%
\global\long\def\hbz{\hat{\bz}}%
\global\long\def\z{\mathrm{z}}%
\global\long\def\y{\mathrm{y}}%
\global\long\def\bxnew{\boldsymbol{y}}%
\global\long\def\bX{\boldsymbol{X}}%
\global\long\def\tbx{\tilde{\bx}}%
\global\long\def\by{\boldsymbol{y}}%
\global\long\def\bY{\boldsymbol{Y}}%
\global\long\def\bZ{\boldsymbol{Z}}%
\global\long\def\bU{\boldsymbol{U}}%
\global\long\def\bn{\boldsymbol{n}}%
\global\long\def\bV{\boldsymbol{V}}%
\global\long\def\bI{\boldsymbol{I}}%
\global\long\def\J{\mathrm{J}}%
\global\long\def\bJ{\mathbf{J}}%
\global\long\def\w{\mathrm{w}}%
\global\long\def\bw{\vt w}%
\global\long\def\bW{\mathbf{W}}%
\global\long\def\balpha{\gvt{\alpha}}%
\global\long\def\bdelta{\boldsymbol{\delta}}%
\global\long\def\bsigma{\gvt{\sigma}}%
\global\long\def\bbeta{\gvt{\beta}}%
\global\long\def\bmu{\gvt{\mu}}%
\global\long\def\btheta{\boldsymbol{\theta}}%
\global\long\def\blambda{\boldsymbol{\lambda}}%
\global\long\def\bgamma{\boldsymbol{\gamma}}%
\global\long\def\bpsi{\boldsymbol{\psi}}%
\global\long\def\bphi{\boldsymbol{\phi}}%
\global\long\def\bpi{\boldsymbol{\pi}}%
\global\long\def\bomega{\boldsymbol{\omega}}%
\global\long\def\bepsilon{\boldsymbol{\epsilon}}%
\global\long\def\btau{\boldsymbol{\tau}}%
\global\long\def\bxi{\boldsymbol{\xi}}%
\global\long\def\realset{\mathbb{R}}%
\global\long\def\realn{\realset^{n}}%
\global\long\def\integerset{\mathbb{Z}}%
\global\long\def\natset{\integerset}%
\global\long\def\integer{\integerset}%

\global\long\def\natn{\natset^{n}}%
\global\long\def\rational{\mathbb{Q}}%
\global\long\def\rationaln{\rational^{n}}%
\global\long\def\complexset{\mathbb{C}}%
\global\long\def\comp{\complexset}%

\global\long\def\compl#1{#1^{\text{c}}}%
\global\long\def\and{\cap}%
\global\long\def\compn{\comp^{n}}%
\global\long\def\comb#1#2{\left({#1\atop #2}\right) }%
\global\long\def\nchoosek#1#2{\left({#1\atop #2}\right)}%
\global\long\def\param{\vt w}%
\global\long\def\Param{\Theta}%
\global\long\def\meanparam{\gvt{\mu}}%
\global\long\def\Meanparam{\mathcal{M}}%
\global\long\def\meanmap{\mathbf{m}}%
\global\long\def\logpart{A}%
\global\long\def\simplex{\Delta}%
\global\long\def\simplexn{\simplex^{n}}%
\global\long\def\dirproc{\text{DP}}%
\global\long\def\ggproc{\text{GG}}%
\global\long\def\DP{\text{DP}}%
\global\long\def\ndp{\text{nDP}}%
\global\long\def\hdp{\text{HDP}}%
\global\long\def\gempdf{\text{GEM}}%
\global\long\def\rfs{\text{RFS}}%
\global\long\def\bernrfs{\text{BernoulliRFS}}%
\global\long\def\poissrfs{\text{PoissonRFS}}%
\global\long\def\grad{\gradient}%
\global\long\def\gradient{\nabla}%
\global\long\def\partdev#1#2{\partialdev{#1}{#2}}%
\global\long\def\partialdev#1#2{\frac{\partial#1}{\partial#2}}%
\global\long\def\partddev#1#2{\partialdevdev{#1}{#2}}%
\global\long\def\partialdevdev#1#2{\frac{\partial^{2}#1}{\partial#2\partial#2^{\top}}}%
\global\long\def\closure{\text{cl}}%
\global\long\def\cpr#1#2{\Pr\left(#1\ |\ #2\right)}%
\global\long\def\var{\text{Var}}%
\global\long\def\Var#1{\text{Var}\left[#1\right]}%
\global\long\def\cov{\text{Cov}}%
\global\long\def\Cov#1{\cov\left[ #1 \right]}%
\global\long\def\COV#1#2{\underset{#2}{\cov}\left[ #1 \right]}%
\global\long\def\corr{\text{Corr}}%
\global\long\def\sst{\text{T}}%
\global\long\def\SST{\sst}%
\global\long\def\ess{\mathbb{E}}%

\global\long\def\Ess#1{\ess\left[#1\right]}%
\newcommandx\ESS[2][usedefault, addprefix=\global, 1=]{\underset{#2}{\ess}\left[#1\right]}%
\global\long\def\fisher{\mathcal{F}}%

\global\long\def\bfield{\mathcal{B}}%
\global\long\def\borel{\mathcal{B}}%
\global\long\def\bernpdf{\text{Bernoulli}}%
\global\long\def\betapdf{\text{Beta}}%
\global\long\def\dirpdf{\text{Dir}}%
\global\long\def\gammapdf{\text{Gamma}}%
\global\long\def\gaussden#1#2{\text{Normal}\left(#1, #2 \right) }%
\global\long\def\gauss{\mathbf{N}}%
\global\long\def\gausspdf#1#2#3{\text{Normal}\left( #1 \lcabra{#2, #3}\right) }%
\global\long\def\multpdf{\text{Mult}}%
\global\long\def\poiss{\text{Pois}}%
\global\long\def\poissonpdf{\text{Poisson}}%
\global\long\def\pgpdf{\text{PG}}%
\global\long\def\wshpdf{\text{Wish}}%
\global\long\def\iwshpdf{\text{InvWish}}%
\global\long\def\nwpdf{\text{NW}}%
\global\long\def\niwpdf{\text{NIW}}%
\global\long\def\studentpdf{\text{Student}}%
\global\long\def\unipdf{\text{Uni}}%
\global\long\def\transp#1{\transpose{#1}}%
\global\long\def\transpose#1{#1^{\mathsf{T}}}%
\global\long\def\mgt{\succ}%
\global\long\def\mge{\succeq}%
\global\long\def\idenmat{\mathbf{I}}%
\global\long\def\trace{\mathrm{tr}}%
\global\long\def\argmax#1{\underset{_{#1}}{\text{argmax}} }%
\global\long\def\argmin#1{\underset{_{#1}}{\text{argmin}\ } }%
\global\long\def\diag{\text{diag}}%
\global\long\def\norm{}%
\global\long\def\spn{\text{span}}%
\global\long\def\vtspace{\mathcal{V}}%
\global\long\def\field{\mathcal{F}}%
\global\long\def\ffield{\mathcal{F}}%
\global\long\def\inner#1#2{\left\langle #1,#2\right\rangle }%
\global\long\def\iprod#1#2{\inner{#1}{#2}}%
\global\long\def\dprod#1#2{#1 \cdot#2}%
\global\long\def\norm#1{\left\Vert #1\right\Vert }%
\global\long\def\entro{\mathbb{H}}%
\global\long\def\entropy{\mathbb{H}}%
\global\long\def\Entro#1{\entro\left[#1\right]}%
\global\long\def\Entropy#1{\Entro{#1}}%
\global\long\def\mutinfo{\mathbb{I}}%
\global\long\def\relH{\mathit{D}}%
\global\long\def\reldiv#1#2{\relH\left(#1||#2\right)}%
\global\long\def\KL{KL}%
\global\long\def\KLdiv#1#2{\KL\left(#1\parallel#2\right)}%
\global\long\def\KLdivergence#1#2{\KL\left(#1\ \parallel\ #2\right)}%
\global\long\def\crossH{\mathcal{C}}%
\global\long\def\crossentropy{\mathcal{C}}%
\global\long\def\crossHxy#1#2{\crossentropy\left(#1\parallel#2\right)}%
\global\long\def\breg{\text{BD}}%
\global\long\def\lcabra#1{\left|#1\right.}%
\global\long\def\lbra#1{\lcabra{#1}}%
\global\long\def\rcabra#1{\left.#1\right|}%
\global\long\def\rbra#1{\rcabra{#1}}%

\begin{abstract}
Ensemble-based adversarial training is a principled approach to achieve
robustness against adversarial attacks. An important technique of
this approach is to control the transferability of adversarial examples
among ensemble members. We propose in this work a simple yet effective
strategy to collaborate among committee models of an ensemble model.
This is achieved via the secure and insecure sets defined for each
model member on a given sample, hence help us to quantify and regularize
the transferability. Consequently, our proposed framework provides
the flexibility to reduce the adversarial transferability as well
as to promote the diversity of ensemble members, which are two crucial
factors for better robustness in our ensemble approach. We conduct
extensive and comprehensive experiments to demonstrate that our proposed
method outperforms the state-of-the-art ensemble baselines, at the
same time can detect a wide range of adversarial examples with a nearly
perfect accuracy.\footnote{Our code is available at: https://github.com/tuananhbui89/Crossing-Collaborative-Ensemble} 
\end{abstract}

\section{Introduction}

Deep neural networks have experienced great success in many disciplines
\citep{goodfellow2016deep}, such as computer vision \citep{he2016deep},
natural language processing and speech processing \citep{vaswani2017attention}.
However, even the state-of-the-art models are reported to be vulnerable
to adversarial attacks \citep{biggio2013evasion,goodfellow2014explaining,szegedy2013intriguing,carlini2017towards,madry2017towards,athalye2018obfuscated},
which is of significant concern given the large number of applications
of deep learning in real-world scenarios. It is thus urgent to develop
deep learning models that are robust against different types of adversarial
attacks. To this end, several adversarial defense methods have been
developed but typically addressing the robustness within a single
model (e.g., \citealp{papernot2016limitations,moosavi2016deepfool,madry2017towards,qin2019adversarial,shafahi2019adversarial}).
To cater for more diverse types of attacks, recent work, notably \citep{he2017adversarial,tramer2017ensemble,strauss2017ensemble,liu2018towards,pang2019improving},
 has shown that ensemble learning can strengthen robustness significantly.

Despite initial success, key principles for ensemble-based adversarial
training (EAT) largely remain open. One crucial challenge is to achieve
minimum `transferability'  between committee members to increase
robustness for the overall ensemble model \citep{papernot2016transferability,liu2016delving,tramer2017ensemble,pang2019improving,kariyappa2019improving}.
In \citep{kariyappa2019improving}, robustness was achieved by aligning
the gradient of committee members to be diametrically opposed, hence
reducing the shared adversarial spaces \citep{tramer2017space}, or
the transferability. However, the method in \citep{kariyappa2019improving}
was designed for black-box attacks, thus still vulnerable to white-box
attacks. Furthermore, attempting to achieve gradient alignment is
unreliable for high-dimensional datasets and it is difficult to extend
for ensemble with more than two committee members. More recently \citep{pang2019improving}
proposed to promote the diversity of non-maximal predictions of the
committee members (i.e., the diversity among softmax probabilities
except the highest ones) to reduce the adversarial transferability
among them. Nonetheless, the central concept of transferability has
not been systematically addressed.

Our proposed work here will first make the concept of adversarial
transferability concrete via the definitions of secure and insecure
sets. To reduce the adversarial transferability and increase the
model diversity, we aim to make the insecure sets of the committee
models as disjoint as possible (i.e., lessening the overlapping of
those regions) and challenge those committee members with divergent
sets of adversarial examples. In addition, we observe that lessening
the adversarial transferability alone is not sufficient to ensure
accurate predictions of the ensemble model because the committee member
that offers inaccurate predictions might dominate the final decisions.
With this in mind, we propose to realize what we call a ``transferring
flow'' by collaborating robustness promoting and demoting operations.
Our key principle to coordinate the promoting and demoting operations
is to promote the prediction of one model on a given adversarial example
and to demote the prediction of another model on this example so as
to maximally lessen the negative impact of the wrong predictions and
ensure the correct predictions of the ensemble model. Moreover, different
from other works \citep{strauss2017ensemble,pang2019improving,kariyappa2019improving}
which only consider adversarial examples of the ensemble model, the
committee members in our ensemble model are exposed to various divergent
adversarial example sets, which inspire them to become gradually more
divergent. Interestingly, by strengthening demoting operations, our
method is capable to assist better detection of adversarial examples.
In brief, our contributions in this work include:
\begin{itemize}
\item We propose a simple but efficient collaboration strategy to reduce
the transferability among ensemble members. 
\item We propose two variants of our method: the robust oriented variant,
which helps to improve the adversarial robustness and the detection
oriented variant, which can detect adversarial examples with high
predictive performance. 
\item We conduct extensive and comprehensive experiments to demonstrate
the improvement of our proposed method over the state-of-the-art defense
methods. 
\item We provide a further understanding of the relationship between the
transferability and the overall robustness in ensemble learning context. 
\end{itemize}

\section{Our Proposed Method}

In this section, we present our ensemble collaboration strategy,
which allows us to collaborate many committee models for improving
the ensemble robustness. We start with the definitions and some key
properties of secure and insecure sets which later support us in devising
promoting and demoting operations for collaborating the committee
models to achieve the ensemble robustness. It is worth noting that
our ensemble strategy is applicable for ensembling an arbitrary number
of committee models; here we focus on presenting the key theories,
principles, and operations for the canonical case of ensembling two
models for better readability.

\subsection{Secure and Insecure Sets \label{subsec:sec_insec_sets}}

Consider a classification problem on a dataset $\mathcal{D}$ with
$M$ classes and a pair $(\bx,\by)$ that represents a data example
$\bx$ and its true label $\by$ which is sampled from the dataset
$\mathcal{D}$. Given a model $f$, the crucial aim of defense is
to make $f$ robust by giving consistently accurate predictions over
a ball, $\mathcal{B}\left(\bx,\epsilon\right)\coloneqq\left\{ \bx':\norm{\bx'-\bx}\leq\epsilon\right\} $
around a benign data example $\bx$, for every possible $\bx$ in
the dataset $\mathcal{D}$ and the distortion boundary $\epsilon$.
To further clarify and motivate our theory, we define
\begin{align*}
\mathcal{B}_{\text{secure}}\left(\bx,\by,f,\epsilon\right) & \coloneqq\left\{ \bx'\in\mathcal{B}\left(\bx,\epsilon\right):\text{argmax}_{i}f_{i}\left(\bx'\right)=\by\right\} ,\\
\mathcal{B}_{\text{insecure}}\left(\bx,\by,f,\epsilon\right) & \coloneqq\left\{ \bx'\in\mathcal{B}\left(\bx,\epsilon\right):\text{argmax}_{i}f_{i}\left(\bx'\right)\neq\by\right\} .
\end{align*}

Intuitively, we define a \emph{secure} set $\mathcal{B}_{\text{secure}}\left(\bx,\by,f,\epsilon\right)$
as the set of elements in the ball $\mathcal{B}\left(\bx,\epsilon\right)$
for which the classifier $f$ makes the correct prediction. In addition,
we define the \emph{insecure} set $\mathcal{B}_{\text{insecure}}\left(\bx,\by,f,\epsilon\right)$
as the set of elements in the ball $\mathcal{B}\left(\bx,\epsilon\right)$
for which $f$  predicts differently from the true label $\by$.
By definition, the secure set is the complement of the insecure set,
and $\mathcal{B}\left(\bx,\varepsilon\right)=\mathcal{B}_{\text{secure}}\left(\bx,\by,f,\varepsilon\right)\bigcup\mathcal{B}_{\text{insecure}}\left(\bx,\by,f,\varepsilon\right)$.
It is clear that the aim of improving adversarial robustness is to
train the classifier $f$ in such the way that $\mathcal{B}_{\text{insecure}}\left(\bx,\by,f,\epsilon\right)$
is either as small as possible (ideally, $\mathcal{B}_{\text{insecure}}\left(\bx,\by,f,\epsilon\right)=\emptyset,\,\forall\bx\in\mathcal{D}$)
or makes an adversary hard to generate adversarial examples in it.
The following simple lemma (see the proof in the supplementary material)
shows the connection between those two kinds of sets and the robustness
of the ensemble model and facilitates the development of our proposed
method.
\begin{lem}
\label{lemma:2}Let us define $f^{en}\left(\cdot\right)=\frac{1}{2}f^{1}\left(\cdot\right)+\frac{1}{2}f^{2}\left(\cdot\right)$
for two given models $f^{1}$ and $f^{2}$. If $f^{1}$ and $f^{2}$
predict an example $\bx$ accurately, we have the following:

i) $\mathcal{B}_{\text{insecure}}\left(\bx,\by,f^{en},\epsilon\right)\subset\mathcal{B}_{\text{insecure}}\left(\bx,\by,f^{1},\epsilon\right)\cup\mathcal{B}_{\text{insecure}}\left(\bx,\by,f^{2},\epsilon\right).$

ii) $\mathcal{B}_{\text{secure}}\left(\bx,\by,f^{1},\epsilon\right)\cap\mathcal{B}_{\text{secure}}\left(\bx,\by,f^{2},\epsilon\right)\subset\mathcal{B}_{\text{secure}}\left(\bx,\by,f^{en},\epsilon\right).$
\end{lem}

\subsection{Dual Collaborative Ensemble\label{subsec:DCE}}

\subsubsection{Transferring Flow. \label{subsec:Transferring-Flow}}

Consider the canonical case of an ensemble consisting of two models:
$f^{en}\left(\cdot\right)=\frac{1}{2}f^{1}\left(\cdot\right)+\frac{1}{2}f^{2}\left(\cdot\right)$,
where $f^{en}$ is the ensemble model and $\{f^{1},f^{2}\}$ is the
set of ensemble committee (or the committee). Based on the definitions
of secure and insecure sets, an arbitrary adversarial example $\bx_{a}$
must lie in one of four subsets as shown in Table \ref{tab:joint-subset}.
Let us further clarify these subsets. In the first subset $S_{11}=\mathcal{B}_{\text{secure}}(\bx,\by,f^{1},\epsilon)\bigcap\mathcal{B}_{\text{secure}}(\bx,\by,f^{2},\epsilon)$,
the example $\bx_{a}$ is predicted correctly by both models, hence
also by the ensemble model $f^{en}$ (Lemma \ref{lemma:2} (ii)).
The subsets $S_{10},S_{01}$ are the intersection of a secure set
of one model and an insecure set of another model, hence an example
of two sets is predicted correctly by one model and incorrectly by
the other. Lastly, in the subset $S_{00}=\mathcal{B}_{\text{insecure}}(\bx,\by,f^{1},\epsilon)\bigcap\mathcal{B}_{\text{insecure}}(\bx,\by,f^{2},\epsilon)$,
both models offer predictions other than the true label, but there
is also no guarantee that their incorrect predictions are in the same
class. There is still a chance that the incorrect prediction in subset
$S_{10}$, $S_{01}$ dominates the correct ones, which leads to the
incorrect prediction on average. Therefore, the insecure region of
the overall ensemble should be related to the union $S_{10}\cup S_{01}\cup S_{00}$
or the total volume (i.e., $|S_{10}|+|S_{01}|+|S_{00}|$) of the subsets
$S_{10},S_{01},S_{00}$. 

As the result, to obtain a robust ensemble model, we need to maintain
the subset $S_{00}$ as small as possible, which is in turn equivalent
to making the insecure regions of the two models as disjoint as much
as possible (i.e., concurred with Lemma \ref{lemma:2} (i)). For the
data points in either $S_{10}$ or $S_{01}$, we need to increase
the chance that the correct predictions dominate the incorrect ones.
Our approach is to encourage adversarial examples inside $S_{00}$
to move to the subsets $S_{10},S_{01}$ during the course of training,
and those of $S_{10},S_{01}$ to move to the subset $S_{11}$. We
term this movement as the \emph{transferring flow}, which is described
in Table \ref{tab:joint-subset}. In what follows, we present how
to implement the transferring flow for our ensemble model. 

\subsubsection{Promoting Adversarial Robustness (PO).\label{subsec:Promoting-and-Demoting}}

We refer to promoting adversarial robustness as an operation to leverage
the information of an example $\bx_{a}^{i}$ (adversarial example
of model $f^{i}$) for improving the robustness of a model $f^{j}$
($i,j$ can be different). There are several adversarial defense methods
that can be applied to promote adversarial robustness, notably \citep{madry2017towards,Zhang2019theoretically,qin2019adversarial}.
In this work, to promote the adversarial robustness of a given adversarial
example $\bx_{a}^{i}$ w.r.t the model $f^{j}$, we use adversarial
training \citep{madry2017towards} by minimizing the cross-entropy
loss w.r.t the true label as $\min\;\mathcal{C}\left(f^{j}(\bx_{a}^{i}),\by\right)$.
After undertaking this PO, $\bx_{a}^{i}$ is expected to move to the
secure set $\mathcal{B}_{\text{secure}}\left(\bx,\by,f^{j},\epsilon\right)$.
We introduce two types of PO: direct PO (dPO) when $i=j$ and crossing
PO (cPO) when $i\neq j$. 

\begin{table}
\caption{Four subsets of the ensemble model and the transferring flow (arrows)\label{tab:joint-subset}}

\centering{}\resizebox{0.45\textwidth}{!}{\centering\setlength{\tabcolsep}{2pt}
\begin{tabular}{c|ccc}
\hline 
 & $\bx_{a}\in\mathcal{B}_{\text{secure}}(\bx,\by,f^{1},\epsilon)$ &  & $\bx_{a}\in\mathcal{B}_{\text{insecure}}(\bx,\by,f^{1},\epsilon)$\tabularnewline
\hline 
$\bx_{a}\in\mathcal{B}_{\text{secure}}(\bx,\by,f^{2},\epsilon)$ & $S_{11}$ & $\Leftarrow$ & $S_{01}$\tabularnewline
 & $\Uparrow$ &  & $\Uparrow$\tabularnewline
$\bx_{a}\in\mathcal{B}_{\text{insecure}}(\bx,\by,f^{2},\epsilon)$ & $S_{10}$ & $\Leftarrow$ & $S_{00}$\tabularnewline
\hline 
\end{tabular}}
\end{table}

\subsubsection{Demoting Adversarial Robustness (DO). }

In contrast to promoting adversarial robustness, we refer to demoting
adversarial robustness as an operation to sacrifice the robustness
of a model for an example $\bx_{a}^{i}$ (adversarial example of model
$f^{i}$). Here, we demote the adversarial robustness of a given adversarial
example $\bx_{a}^{i}$ w.r.t the model $f^{j}$ by $\max\;\mathcal{H}\left(f^{j}(\bx_{a}^{i})\right)$
where $\mathcal{H}$ is the entropy. Without any further knowledge,
the prediction is likely uniformly distributed, hence the example
$\bx_{a}^{i}$ likely falls into the insecure set $\mathcal{B}_{\text{insecure}}\left(\bx,\by,f^{j},\epsilon\right)$
instead of the secure set $\mathcal{B}_{\text{secure}}\left(\bx,\by,f^{j},\epsilon\right)$.

\subsubsection{Collaboration of the Promoting and Demoting Operations.}

We now present how to coordinate PO/DO to enforce the transferring
flow for enhancing the adversarial robustness of the ensemble model
in the canonical case of a committee of two members $\{f^{1},f^{2}\}$,
parameterized by $\theta_{1}$ and $\theta_{2}$. Let $\bx_{a}^{1}$
and $\bx_{a}^{2}$ be white-box adversarial examples of $f^{1}$ and
$f^{2}$ respectively. With a strong adversary, we can assume that
$\bx_{a}^{1}\in\mathcal{B}_{\text{insecure}}(\bx,\by,f^{1},\epsilon)$
(i.e., $\bx_{a}^{1}\in S_{01}\cup S_{00}$) and $\bx_{a}^{2}\in\mathcal{B}_{\text{insecure}}(\bx,\by,f^{2},\epsilon)$
(i.e., $\bx_{a}^{2}\in S_{10}\cup S_{00}$). For ease of comprehensibility,
we present the treatment for $\bx_{a}^{1}$ and the same treatment
is applied to $\bx_{a}^{2}$. To strengthen model $f^{1}$, we always
use $\bx_{a}^{1}$ to promote the robustness of model $f^{1}$ by
minimizing the cross-entropy loss $\mathcal{C}\left(f^{1}(\bx_{a}^{1}),\by\right)$
(i.e., flow $S_{01}\Rightarrow S_{11}$ or $S_{00}\Rightarrow S_{10}$).
Meanwhile, we consider two cases of $\bx_{a}^{1}$ w.r.t model $f^{2}$:
i) being correctly predicted by $f^{2}$ (i.e., $\bx_{a}^{1}\in S_{01}$)
and ii) being incorrectly predicted by $f^{2}$ (i.e., $\bx_{a}^{1}\in S_{00}$).
For the first case, we use $\bx_{a}^{1}$ to promote model $f^{2}$
to make sure $\bx_{a}^{1}$ stays in the secure set of model $f^{2}$
(i.e., $S_{11}\cup S_{01}$). For the second case, we demote $\bx_{a}^{1}$
w.r.t $f^{2}$ by maximizing the entropy $\mathcal{H}\left(f^{2}(\bx_{a}^{1})\right)$
in order to keep $\bx_{a}^{1}$ in the insecure set of model $f^{2}$
(i.e., $S_{10}\cup S_{00}$). 

Therefore, with the collaboration of two models $f^{1}$ and $f^{2}$
on the same example $\bx_{a}^{1}$, we deploy either flow $S_{01}\Rightarrow S_{11}$
or $S_{00}\Rightarrow S_{10}$ depending on the scenario of $\bx_{a}^{1}$.
It is worth noting that DO  encourages $f^{2}(\bx_{a}^{1})$ to be
close to the uniform prediction, hence causing a minimal effect on
the ensemble prediction $f^{en}\left(\bx_{a}^{1}\right)$. As a consequence,
$f^{en}\left(\bx_{a}^{1}\right)=\frac{1}{2}\left(f^{1}\left(\bx_{a}^{1}\right)+f^{2}\left(\bx_{a}^{1}\right)\right)$
is dominated by $f^{1}\left(\bx_{a}^{1}\right)$, which likely offers
a correct prediction via the corresponding PO: $\min\,\mathcal{C}\left(f^{1}(\bx_{a}^{1}),\by\right)$.
We summarize the PO/DO to deploy the transferring flow in Table \ref{tab:actions}.

\begin{table}
\caption{Promoting and demoting operations for the transferring flow\label{tab:actions}}

\centering{}\resizebox{0.3\textwidth}{!}{\centering\setlength{\tabcolsep}{2pt}
\begin{tabular}{c|c|c}
\hline 
Scenario & $f^{1}$ & $f^{2}$\tabularnewline
\hline 
$\bx_{a}^{1}\in S_{01}$ & $\text{min}\;\mathcal{C}\left(f^{1}(\bx_{a}^{1}),\by\right)$ & $\text{min}\;\mathcal{C}\left(f^{2}(\bx_{a}^{1}),\by\right)$ \tabularnewline
$\bx_{a}^{1}\in S_{00}$ & $\text{min}\;\mathcal{C}\left(f^{1}(\bx_{a}^{1}),\by\right)$  & $\text{max}\;\mathcal{H}\left(f^{2}(\bx_{a}^{1})\right)$ \tabularnewline
$\bx_{a}^{2}\in S_{10}$ & $\text{min}\;\mathcal{C}\left(f^{1}(\bx_{a}^{2}),\by\right)$ & $\text{min}\;\mathcal{C}\left(f^{2}(\bx_{a}^{2}),\by\right)$ \tabularnewline
$\bx_{a}^{2}\in S_{00}$ & $\text{max}\;\mathcal{H}\left(f^{1}(\bx_{a}^{2})\right)$ & $\text{min}\;\mathcal{C}\left(f^{2}(\bx_{a}^{2}),\by\right)$ \tabularnewline
\hline 
\end{tabular}}
\end{table}

The objective functions for model $f^{1}$ and $f^{2}$ to deploy
the transferring flow are: 

{\small{}
\begin{align}
\mathcal{L}(\bx,\by,\theta_{1}) & =\mathcal{C}\left(f^{1}(\bx),\by\right)+\mathcal{C}\left(f^{1}(\bx_{a}^{1}),\by\right)\nonumber \\
 & +\lambda_{pm}\mathbb{I}\left(f^{1}(\bx_{a}^{2}),\by\right)\mathcal{C}(f^{1}(\bx_{a}^{2}),\by)\nonumber \\
 & -\lambda_{dm}\left(1-\mathbb{I}\left(f^{1}(\bx_{a}^{2}),\by\right)\right)\mathcal{H}\left(f^{1}(\bx_{a}^{2})\right),\\
\mathcal{L}(\bx,\by,\theta_{2}) & =\mathcal{C}\left(f^{2}(\bx),\by\right)+\mathcal{C}\left(f^{2}(\bx_{a}^{2}),\by\right)\nonumber \\
 & +\lambda_{pm}\mathbb{I}\left(f^{2}(\bx_{a}^{1}),\by\right)\mathcal{C}(f^{2}(\bx_{a}^{1}),\by)\nonumber \\
 & -\lambda_{dm}\left(1-\mathbb{I}\left(f^{2}(\bx_{a}^{1}),\by\right)\right)\mathcal{H}\left(f^{2}(\bx_{a}^{1})\right).
\end{align}
}{\small\par}

where $\lambda_{pm}$ and $\lambda_{dm}$ are the hyper-parameters
for promoting and demoting effects, respectively, and $\mathbb{I}\left(f^{1}(\bx_{a}^{2}),\by\right)$
is the indicator to indicate whether $\bx_{a}^{2}$ is predicted correctly
(i.e., $\mathbb{I}=1$, hence $\bx_{a}^{2}\in S_{10}$) or incorrectly
(i.e., $\mathbb{I}=0$, hence $\bx_{a}^{2}\in S_{00}$) by $f^{1}$,
which helps to switch on/off the cPO/DO for model $f^{1}$.

For the final objective function, we approximate the hard indicator
$\mathbb{I}\left(f^{1}(\bx_{a}^{2}),\by\right)$ by the soft version
$f_{\by}^{1}(\bx_{a}^{2})=p\left(\by\mid\bx_{a}^{2},f^{1}\right)$,
which represents the probability the model $f^{1}$ assigning $\bx_{a}^{2}$
to the label $\by$. We hence arrive at the following objective functions
for both $f^{1}$ and $f^{2}$, respectively. 

{\small{}
\begin{align}
\mathcal{L}(\bx,\by,\theta_{1}) & =\mathcal{C}\left(f^{1}(\bx),\by\right)+\mathcal{C}\left(f^{1}(\bx_{a}^{1}),\by\right)\nonumber \\
 & +\lambda_{pm}f_{y}^{1}(\bx_{a}^{2})\mathcal{C}(f^{1}(\bx_{a}^{2}),\by)\nonumber \\
 & -\lambda_{dm}\left(1-f_{y}^{1}(\bx_{a}^{2})\right)\mathcal{H}\left(f^{1}(\bx_{a}^{2})\right),\\
\mathcal{L}(\bx,\by,\theta_{2}) & =\mathcal{C}\left(f^{2}(\bx),\by\right)+\mathcal{C}\left(f^{2}(\bx_{a}^{2}),\by\right)\nonumber \\
 & +\lambda_{pm}f_{y}^{2}(\bx_{a}^{1})\mathcal{C}(f^{2}(\bx_{a}^{1}),\by)\nonumber \\
 & -\lambda_{dm}\left(1-f_{y}^{2}(\bx_{a}^{1})\right)\mathcal{H}\left(f^{2}(\bx_{a}^{1})\right).
\end{align}
}{\small\par}

We note that in our implementation, the soft indicators $f_{y}^{1}(\bx_{a}^{2})$
and $f_{y}^{2}(\bx_{a}^{1})$ are used as values by performing a stopping
gradient to prevent the back-propagation process to go inside them
for further updating $f^{1}$ and $f^{2}$. 

\subsection{Crossing Collaborative Ensemble \label{subsec:Crossing-Collaborative-Ensemble}}

We now extend our collaboration strategy to enable us to ensemble
many individual members, which we term as a \emph{Crossing Collaborative
Ensemble (CCE)}. Specifically, given an ensemble of $N$ members $f^{en}\left(\cdot\right)=\frac{1}{N}\sum_{n=1}^{N}f^{n}\left(\cdot\right)$
parameterized by $\theta_{n}$, the loss function for a model $f^{n},n\in[1,N]$
as follow:

{\small{}
\begin{align}
\mathcal{L}^{n}(\bx,\by,\theta_{n}) & =\mathcal{{C}}\left(f^{n}(\bx),\by\right)+\mathcal{{C}}\left(f^{n}(\bx_{a}^{n}),\by\right)\nonumber \\
 & +\frac{1}{N-1}\underset{i\neq n}{\sum}\Biggl(\lambda_{pm}f_{y}^{n}(\bx_{a}^{i})\mathcal{C}(f^{n}(\bx_{a}^{i}),\by)\nonumber \\
 & -\lambda_{dm}\left(1-f_{y}^{n}(\bx_{a}^{i})\right)\mathcal{{H}}\left(f^{n}(\bx_{a}^{i})\right)\Biggr).
\end{align}
}{\small\par}

It appears from the above loss that we encourage each individual model
to (i) minimize the loss of the adversarial example itself for improving
its robustness (dPO) and (ii) promoting or demoting its robustness
(cPO/DO) with other adversarial examples depending on the soft indicator. 

\paragraph{Connections to Traditional Ensemble Learning.}

Firstly, in our method, N members $\{f^{n}\}$ are reinforced with
the joint of $N+1$ data sources: clean data $\{\bx\}$ and $N$ adversarial
examples $\{\bx_{a}^{n}\}_{n=1}^{N}$. However, depending on different
scenarios, they have the same task (PO-PO) or opposite tasks (PO-DO)
on the same adversarial set $\{\bx_{a}^{n}\}$. Our approach can be
linked to the bagging technique in the literature, in which each classifier
was trained on different sets of data. Secondly, by assigning opposite
tasks for ensemble members, our method produces a negative correlation
which was described in \citep{liu1999ensemble,kuncheva2003measures,bagnall2017training}.
It has been claimed that negative relationship among ensemble members
can further improve the ensemble accuracy better than the independent
correlation.

\section{Experiments}

In this section, we first introduce the experimental setting for adversarial
defenses and attackers followed by an extensive evaluation to compare
our method with state-of-the-art adversarial defenses. We show that
our method surpasses these methods for common benchmark datasets.
Next, we provide an ablation study to understand the transferability
among ensemble members of adversarial examples. Finally, we show that
 our method not only detects adversarial examples accurately and
consistently but also predicts benign examples with a significant
improvement.  

\subsection{Experimental Setting}

\subsubsection{General Setting. \label{subsec:general_setting}}

We use CIFAR10 and CIFAR100 as the benchmark datasets in our experiment.\footnote{Recently, \citep{tsipras2020imagenet} found the labeling issue in
the ImageNet dataset, which highly affects the fairness of robustness
evaluation on this dataset.} Both datasets have 50,000 training images and 10,000 test images.
The inputs were normalized to $[0,1]$. We apply random horizontal
flips and random shifts with scale $10\%$ for data augmentation as
used in \citep{pang2019improving}. We use both standard CNN architecture
and ResNet architecture \citep{he2016deep} in our experiment. The
architecture and training setting for each dataset are provided in
our supplementary material. 

\subsubsection{Crafting Adversarial Examples for Defenders.}

In our experiments, we use PGD $\{k,\epsilon,\eta,l_{\infty}\}$ as
the common adversary to generate adversarial examples for the adversarial
training of all defenders where $k$ is the iteration steps, $\epsilon$
is the distortion bound and $\eta$ is the step size. Specifically,
the configuration for the CIFAR10 dataset is $k=10,\epsilon=8/255,\eta=2/255$
and that for the CIFAR100 dataset is $k=10,\epsilon=0.01,\eta=0.001$.
For the CIFAR10 dataset with ResNet architecture, we use the same
setting in \citep{pang2019improving} which is $k=10,\epsilon\sim U(0.01,0.05),\eta=\epsilon/10$.

\subsubsection{Baseline Methods.}

Because the model capacity has significant impact on the inference
performance, therefore, for a fair comparison, we compare our method
with the start-of-the-art ensemble-based method, i.e., ADV-EN \citep{madry2017towards}
and ADP \citep{pang2019improving}, which have the same number of
committee members and also the member's architecture. More specifically,
ADV-EN is the variant of PGD adversarial training method (ADV) in
the context of ensemble learning, in which the entire ensemble model
is treated as one unified model applied with adversarial training.
We also compare with the ADV method which is adversarial training
on a single model. For ADP, we choose the best setting $ADP_{2,0.5}$
with adversarial version, which was reported in the paper \citep{pang2019improving},
and use the official code.\footnote{https://github.com/P2333/Adaptive-Diversity-Promoting}

Throughout our experiments, we use two variants of our method: (i)
Robustness Mode (i.e., CCE-RM) for which we set $\lambda_{pm}=\lambda_{dm}=1$
and (ii) Detection Mode (i.e., CCE-DM) for which we disable cPO ($\lambda_{pm}=0$)
and strengthen DO (i.e., $\lambda_{dm}=5$).

\subsubsection{Attack Setting.}

We use different state-of-the-art attacks to evaluate the defense
methods including: 

(i) \textbf{Gradient based attacks} (with \emph{cleverhans}\footnote{https://github.com/tensorflow/cleverhans}
lib). We use PGD \citep{madry2017towards}, the Basic Iterative Method
(BIM) \citep{kurakin2016a} and the Momentum Iterative Method (MIM)
\citep{dong2018boosting}. They share the same hyper-parameters configuration,
i.e., $\{k,\epsilon,\eta\}$, which is described in each individual
experiment. 

(ii) \textbf{B\&B attack} \citep{brendel2019accurate} (with \emph{foolbox}\footnote{https://foolbox.readthedocs.io/en/stable/}
lib) which is a decision based attack. We argue that the B\&B attack
setting in the paper of \citep{tramer2020adaptive} may not be appropriate
to evaluate the ADP method. It is because the ADP method used PGD
$(\epsilon\sim U(0.01,0.05),k=10)$ for its adversarial training,
while B\&B attack used PGD $(\epsilon=0.15,k=20)$ as an initialized
attack which is much stronger than the defense capacity. More specifically,
the initialized PGD attack alone can reduce the accuracy to 0.1\%.
Therefore, B\&B attack contributes very little to the final attack
performance. To have a fair evaluation, we use two initialized attacks
with lower strength: PGD1 $(\epsilon=8/255,\eta=2/255,k=20)$ and
PGD2 $(\epsilon=16/255,\eta=2/255,k=20)$ then apply B\&B attack with
100 steps and repeat for three times. It is worth noting that, PGD2
is still much stronger than the defense capacity, however, we use
this setting to mimic the evaluation in the paper of \citep{tramer2020adaptive}.

(iii) \textbf{Auto-Attack} \citep{croce2020reliable} (with the official
implementation\footnote{https://github.com/fra31/auto-attack}) which
is an ensemble based attack. We use $\epsilon=8/255$ for the CIFAR10
dataset and $\epsilon=0.01$ for the CIFAR100 dataset, both with standard
version which is an ensemble of four different attacks.

(iv) \textbf{SPSA attack} \citep{uesato2018adversarial} (with \emph{cleverhans}
lib) which is a gradient-free optimization method. We use $\epsilon=8/255$
for the CIFAR10 dataset and $\epsilon=0.01$ for the CIFAR100 dataset,
both with 50 steps\textbf{.}

The distortion metric we use in our experiments is $l_{\infty}$ for
all measures. We use the full test set for the attacks (i) and 1000
test samples for the attacks (ii-iv).

\subsection{Robustness Evaluation }

We conduct extensive experiments on the CIFAR10 and CIFAR100 datasets
to compare our method with the other methods. We consider the ensemble
of both two and three committee members (denoted by a subscript number
in each method). It can be observed from the experimental results
in Table {[}\ref{tab:result-resnet}, \ref{tab:result-c10}, \ref{tab:result-c100}{]}
that: 

(i) There is a gap of 2\%$\sim$3\% when comparing $\text{ADV-EN}_{3}$
with $\text{ADV}_{1}$ showing that increasing model capacity (by
increasing number of ensemble member) can improve the robustness of
the model. 

(ii) There is a gap of 3\%$\sim$4\% between $\text{ADP}_{3}$ and
$\text{ADV}_{1}$, and especially, a gap of 7\%$\sim$8\% when comparing
our $\text{CCE-RM}_{3}$ with $\text{ADV}_{1}$, which shows the potential
of the ensemble learning to tackle with the adversarial attacks. 

(iii) With the same model capacity, our CCE-RM is consistently the
best with all attacks and in some attacks, ours surpasses other baselines
in a large margin (4\%$\sim$5\%). 

(iv) There is a gap of $3\%$ between $\text{CCE-RM}_{3}$ and $\text{CCE-RM}_{2}$,
which is larger than the gap of $1\%$ between $\text{ADP}_{3}$ and
$\text{ADP}_{2}$ or that of $\text{ADV-EN}_{3}$ and $\text{ADV-EN}_{2}$,
showing that our method collaborates members better and gets more
benefit from ensembling more committee members. 

\begin{table}
\caption{Robustness evaluation on the CIFAR10 dataset with ResNet architecture.
For the PGD attack, we use $\epsilon=8/255,\eta=2/255$. ({*}) The
low robust accuracies (even with standard method $\text{ADV}$) because
the attack strength of PGD2 is double of the defense capacity, which
makes the adversarial examples to be recognizable. \label{tab:result-resnet}}

\centering{}\resizebox{0.45\textwidth}{!}{\centering\setlength{\tabcolsep}{2pt}
\begin{tabular}{cccccc>{\centering}p{0.1mm}ccc}
\hline 
Attack & $\text{ADV}_{1}$ &  & $\text{ADV-EN}_{2}$ & $\text{ADP}_{2}$ & $\text{CCE-RM}_{2}$ &  & $\text{ADV-EN}_{3}$ & $\text{ADP}_{3}$ & $\text{CCE-RM}_{3}$\tabularnewline
\hline 
Non-att (Nat. acc.) & 83.9 &  & 85.3 & \textbf{85.3} & 84.5 &  & 86.1 & \textbf{86.2} & 84.9\tabularnewline
PGD $k=250$ & 41.4 &  & 42.8 & 44.2 & \textbf{45.8} &  & 43.8 & 45.1 & \textbf{48.6}\tabularnewline
BIM $k=250$ & 41.5 &  & 42.9 & 44.1 & \textbf{45.8} &  & 44.0 & 45.2 & \textbf{48.8}\tabularnewline
MIM $k=250$ & 41.9 &  & 43.3 & 44.8 & \textbf{46.3} &  & 44.5 & 45.7 & \textbf{49.1}\tabularnewline
B\&B (wPGD1) & 37.0 &  & 38.3 & 37.3 & \textbf{42.2} &  & 39.3 & 38.3 & \textbf{44.2}\tabularnewline
B\&B (wPGD2){*} & 4.9 &  & 2.9 & 3.9 & \textbf{6.0} &  & 4.2 & 4.3 & \textbf{7.1}\tabularnewline
SPSA & 50.0 &  & 53.5 & 52.8 & \textbf{56.2} &  & 53.8 & 53.9 & \textbf{56.6}\tabularnewline
Auto-Attack & 16.1 &  & 18.5 & 17.3 & \textbf{18.8} &  & 18.4 & 17.6 & \textbf{20.8}\tabularnewline
\hline 
\end{tabular}}
\end{table}

The effectiveness of adversarial training method depends on the diversity
(or the hardness) of the adversarial examples \citep{madry2017towards}.
Fort et al. (2019) found that differently initializing members' parameters,
even with the same training data, can end up with different local
optimal in the solution space. Therefore, the potential of ensemble
learning (in the remark ii) can be explained by the fact that the
adversarial space of an ensemble model $\mathcal{B}_{\text{insecure}}\left(\bx,\by,f^{en},\epsilon\right)$
is more diverse than that of a single model $\mathcal{B}_{\text{insecure}}\left(\bx,\by,f,\epsilon\right)$. 

\begin{table}
\caption{Robustness evaluation on the CIFAR10 dataset with standard CNN architecture.
We use $\epsilon=8/255,\eta=2/255$. Note that $mul\mathcal{A}$ represents
for multiple-targeted attack by adversary $\mathcal{A}$. \label{tab:result-c10}}

\centering{}\resizebox{0.45\textwidth}{!}{\centering\setlength{\tabcolsep}{2pt}
\begin{tabular}{cccccc>{\centering}p{0.1mm}ccc}
\hline 
Attack & $\text{ADV}_{1}$ &  & $\text{ADV-EN}_{2}$ & $\text{ADP}_{2}$ & $\text{CCE-RM}_{2}$ &  & $\text{ADV-EN}_{3}$ & $\text{ADP}_{3}$ & $\text{CCE-RM}_{3}$\tabularnewline
\hline 
Non-att (Nat. acc.) & 75.7 &  & 76.0 & 75.9 & \textbf{76.0} &  & \textbf{76.7} & 76.6 & 75.7\tabularnewline
PGD $k=100$ & 38.0 &  & 39.7 & 42.2 & \textbf{44.7} &  & 40.8 & 43.9 & \textbf{46.8}\tabularnewline
BIM $k=100$ & 38.2 &  & 39.7 & 42.2 & \textbf{44.9} &  & 40.8 & 43.8 & \textbf{46.8}\tabularnewline
MIM $k=100$ & 38.5 &  & 40.5 & 42.4 & \textbf{45.4} &  & 41.3 & 44.2 & \textbf{47.2}\tabularnewline
mul-PGD $k=20$ & 26.0 &  & 27.7 & 27.8 & \textbf{31.9} &  & 28.3 & 32.4 & \textbf{36.9}\tabularnewline
mul-BIM $k=20$ & 25.9 &  & 27.2 & 27.2 & \textbf{31.6} &  & 27.7 & 29.8 & \textbf{34.1}\tabularnewline
mul-MIM $k=20$ & 26.2 &  & 28.1 & 28.3 & \textbf{32.3} &  & 29.0 & 30.7 & \textbf{34.6}\tabularnewline
SPSA & 40.6 &  & 44.3 & 41.5 & \textbf{45.2} &  & 45.1 & 46.1 & \textbf{47.5}\tabularnewline
Auto-Attack & 25.1 &  & 25.0 & 24.4 & \textbf{29.9} &  & 25.5 & 28.1 & \textbf{31.9}\tabularnewline
\hline 
\end{tabular}}
\end{table}

\begin{table}
\caption{Robustness evaluation on the CIFAR100 dataset with standard CNN architecture.
We use $\epsilon=0.01,\eta=0.001$. Note that $mul\mathcal{A}$ represents
for multiple-targeted attack by adversary $\mathcal{A}$.\label{tab:result-c100}}

\centering{}\resizebox{0.45\textwidth}{!}{\centering\setlength{\tabcolsep}{2pt}
\begin{tabular}{cccccc>{\centering}p{0.1mm}ccc}
\hline 
Attack & $\text{ADV}_{1}$ &  & $\text{ADV-EN}_{2}$ & $\text{ADP}_{2}$ & $\text{CCE-RM}_{2}$ &  & $\text{ADV-EN}_{3}$ & $\text{ADP}_{3}$ & $\text{CCE-RM}_{3}$\tabularnewline
\hline 
Non-att (Nat. acc.) & 40.8 &  & 41.4 & 48.0 & \textbf{53.4} &  & 40.8 & 52.6 & \textbf{54.4}\tabularnewline
PGD $k=100$ & 26.8 &  & 29.7 & 30.9 & \textbf{35.3} &  & 32.8 & 36.2 & \textbf{39.5}\tabularnewline
BIM $k=100$ & 26.9 &  & 29.1 & 31.0 & \textbf{35.2} &  & 32.8 & 36.2 & \textbf{39.4}\tabularnewline
MIM $k=100$ & 27.0 &  & 29.0 & 30.8 & \textbf{35.3} &  & 32.9 & 36.1 & \textbf{39.6}\tabularnewline
mul-PGD $k=20$ & 16.4 &  & 15.8 & 20.1 & \textbf{24.2} &  & 16.6 & 24.8 & \textbf{28.4}\tabularnewline
mul-BIM $k=20$ & 15.9 &  & 15.5 & 19.4 & \textbf{23.7} &  & 16.3 & 24.5 & \textbf{28.1}\tabularnewline
mul-MIM $k=20$ & 16.7 &  & 16.1 & 20.3 & \textbf{24.1} &  & 16.8 & 25.1 & \textbf{28.6}\tabularnewline
SPSA & 25.6 &  & 25.5 & 24.1 & \textbf{31.8} &  & 26.0 & 32.5 & \textbf{35.0}\tabularnewline
Auto-Attack & 15.3 &  & 15.1 & 14.8 & \textbf{21.9} &  & 15.8 & 23.0 & \textbf{25.9}\tabularnewline
\hline 
\end{tabular}}
\end{table}

Our advantages over others (in the remark iii, iv) can be explained
by the fact that our proposed method encourages the diversity of its
committee members. Specifically it can be elaborated on with the following
three key points. Firstly, while other ensemble-based defenses use
the adversarial examples of the entire ensemble $\bx_{a}^{en}\sim\mathcal{B}_{\text{insecure}}\left(\bx,\by,f^{en},\epsilon\right)$,
our method makes use of the broader joint adversarial space $\bx_{a}^{i}\sim\mathcal{B}_{\text{insecure}}\left(\bx,\by,f^{i},\epsilon\right)$
(Lemma \ref{lemma:2} (i)). Secondly, each member has different loss
landscape \citep{fort2019deep}, in addition with the randomness of
an adversary (e.g., random starting points in PGD), each member has
its individual adversarial set (partly collapsed as shown in the next
experiment). Therefore, similar with the bagging technique, by promoting
each member with its adversarial examples independently, we can increase
the diversity of the joint adversarial space. Last but not least,
inspired from traditional ensemble learning \citep{liu1999ensemble},
by elegantly collaborating PO and DO, we encourage the negative correlation
among ensemble members, therefore, further improve the diversity of
the joint adversarial space.  

\subsection{Transferability among Ensemble Members }

The transferability is a phenomenon when adversarial examples generated
to attack a specific model also mislead other models trained for the
same task. In the ensemble learning context, adversarial examples
which are transferred well among members will likely fool the entire
ensemble. Therefore, reducing the transferability among members is
a principled approach to achieve better robustness as claimed in the
previous works \citep{pang2019improving,kariyappa2019improving}.
In this sub-section, we provide a further understanding of the transferability
to the overall robustness and show the impact of the transferring
flow. 

We first summarize the experiments setting. The experiments are conducted
on the CIFAR10 dataset with an ensemble of two members under PGD attack
with $k=20,\epsilon=8/255,\eta=2/255$. The results are reported in
Table \ref{tab:transfer}. CCE-Base is our model which disables the
crossing PO and DO by setting $\lambda_{pm}=\lambda_{dm}=0$. $a^{(i,j)}$
represents for the robust accuracy when adversarial examples $\{x_{a}^{i}\}$
attack model $f^{j}$. $\left|S\right|$ shows the cardinality of
a subset $S$, i.e., the percentage of the images that go into the
subset $S$, which can be one of $\{S_{11},S_{01},S_{10},S_{00}\}$.
From the definition of the transferability as mentioned above, to
measure the transferability of adversarial examples $\{\bx_{a}^{i}\}$,
we can compute the accuracy difference of model $f^{i}$ and $f^{j},j\neq i$
against the same attack $\{\bx_{a}^{i}\}$. The smaller gap implies
that adversarial examples $\{\bx_{a}^{i}\}$ are more transferable.
The overall transferability of an ensemble method can be evaluated
by the sum the accuracy differences over all its members, i.e., $T=a^{(1,2)}-a^{(1,1)}+a^{(2,1)}-a^{(2,2)}$. 

We would like to emphasize some following important empirical observations
(Table \ref{tab:transfer}): 

1) \textbf{The impact of the transferring flow. }It can be observed
that the cardinality $\left|S_{11}\right|$ in CCE-RM $(39.9\%)$
is larger than that in CCE-Base ($36.1\%)$, while the cardinality
$\left|S_{01}\right|,\left|S_{10}\right|,\left|S_{00}\right|$ is
smaller than those in CCE-Base which serves as evidence that the adversarial
examples are successfully transferred from subsets $S_{10},S_{01},S_{00}$
to subset $S_{11}$ as we expect. This helps improve the overall robustness
of the ensemble model from $43.3\%$ for CCE-Base to $45.5\%$ for
CCE-RM.

2) \textbf{The transferable space is just a subset of the adversarial
space.} By definition, the subset $S_{00}$ consists of adversarial
examples which fools both models $f^{1},f^{2}$, therefore, $S_{00}$
represents for the transferable space of the ensemble model $f^{en}$.
In fact, the cardinality of $\left|S_{00}\right|$ is smaller than
the insecure region of the ensemble model $f^{en}$ (i.e., the total
classification error $100\%-a^{(en,en)}$) in all methods showing
that the transferable space cannot represent for the insecure region
of the ensemble model $f^{en}$, and the former is just the subset
of the latter.

3) \textbf{Reducing transferability among ensemble members is not
enough to improve adversarial robustness. }In fact, the transferability
metric $T$ for CCE-RM is $33.7\%$ which is much smaller than those
for ADP and ADV-EN ($59.3\%$ and $65.5\%$, respectively). The smaller
value of $T$ shows that the adversarial examples $\{\bx_{a}^{1}\},\{\bx_{a}^{2}\}$
in our method are more transferable than those in ADV-EN and ADP.
However, the fact that the overall robustness of our method is significantly
better evidently shows that \emph{transferability is not the only
factor for improving the robustness. }This is because the robustness
of each individual member under a direct attack (i.e., $a^{(1,1)}$
or $a^{(2,2)}$) is much lower than our method. In addition, the cardinality
$\left|S_{11}\right|$ in our method is $39.9\%$ which is much bigger
than those in ADV-EN ($24.0\%$) and ADP ($25.7\%$).

We provide two additional metrics which are (i) $nT=100\%-a^{(en,en)}-\left|S_{00}\right|$
to measure the cardinality of \emph{adversarial examples set which
successful attack model $f^{en}$ but non transferable among $f^{1},f^{2}$}
and (ii) $a_{single}=a^{(en,en)}-\left|S_{11}\right|$ to measure
the cardinality of \emph{adversarial examples set which are correctly
predicted by only one model either $f^{1}$ or $f^{2}$ but still
being correctly predicted by model $f^{en}$}. The comparison on the
metric $nT$ in Table \ref{tab:transfer} shows that most of successful
adversarial examples in our method are predicted incorrectly by both
members. While the comparison on the metric $a_{single}$ shows that
most of unsuccessful adversarial examples in our method are predicted
correctly by both members. The two comparisons demonstrate that our
method have better robustness than other methods because (i) the adversarial
examples have to fool both ensemble members for a successful attack
and (ii) our ensemble model can predict correctly by both members
which explains the higher performance.

The remarks (2, 3) further imply that:

\emph{An ensemble model cannot be secure against white-box attacks
unless its members are robust against direct attacks (even they are
secure against transferred attacks)}. 

This hypothesis provides more understanding of the correlation between
the transferability and the overall robustness of an ensemble model. 

\begin{table}
\begin{centering}
\caption{Evaluation on the transferability among ensemble members on the CIFAR10
dataset. $\{T,nT,a_{single}\}$ are the metrics of interest. \label{tab:transfer}}
\par\end{centering}
\centering{}\resizebox{0.45\textwidth}{!}{\centering\setlength{\tabcolsep}{2pt}
\begin{tabular}{cccccccc|ccc}
\hline 
Model & $a^{(en,en)}$ & $a^{(1,1)}$ & $a^{(2,2)}$ & $\left|S_{11}\right|$ & $\left|S_{01}\right|$ & $\left|S_{10}\right|$ & $\left|S_{00}\right|$ & $T$ & $nT$ & $a_{single}$\tabularnewline
\hline 
ADV-EN & 40.7 & 31.1 & 33.2 & 24.0 & 17.0 & 13.0 & 46.0 & 65.5 & 13.3 & 16.7\tabularnewline
ADP & 42.9 & 31.0 & 33.1 & 25.7 & 13.1 & 11.7 & 49.5 & 59.3 & 7.6 & 17.2\tabularnewline
CCE-RM & 45.5 & 41.7 & 41.4 & 39.9 & 5.2 & 5.5 & 49.5 & 33.7 & 5.0 & 5.6\tabularnewline
CCE-Base & 43.3 & 40.3 & 40.5 & 36.1 & 6.5 & 7.2 & 50.3 & 36.1 & 6.4 & 7.2\tabularnewline
\hline 
\end{tabular}}
\end{table}

\subsection{Improving Natural Accuracy and Adversarial Detectability }

The parameter $\lambda_{pm}(\lambda_{dm})$ controls the level of
the agreement (disagreement) of models $\{f^{i}\},i\in[1,N]$ and
model $f^{j},j\neq i$ on the same adversarial example $x_{a}^{j}$.
By disabling the crossing PO ($\lambda_{pm}=0$) and strengthening
DO (i.e., $\lambda_{dm}=5)$, our method encourages the disagreement
among members on the same data example, therefore, increases the negative
correlation among them. This setting of CCE-DM leads to two important
properties, which are empirically proved by the experiments below. 

\paragraph{Improving Natural Accuracy. }

We compare natural accuracies of two variants: CCE-RM and CCE-DM against
the baselines. Table \ref{tab:clean_acc} shows that CCE-DM significantly
improves natural accuracy of the ensemble model by a large margin.
In traditional ensemble learning, the key ingredient to improve natural
performance is making ensemble members more diverse \citep{kuncheva2003measures}.
By disabling the crossing PO and strengthening DO, CCE-DM variant
enforces the diversity more strictly, which explains the improvement
of the natural performance. This result demonstrates the promising
usage of adversarial examples to improve the traditional ensemble
learning. 

\begin{table}[h]
\caption{Comparison of the natural performance on the CIFAR10 dataset (the
subscript number denotes the number members).\label{tab:clean_acc}}

\centering{}\resizebox{0.30\textwidth}{!}{\centering\setlength{\tabcolsep}{2pt}
\begin{tabular}{ccccc}
\hline 
Model & $\text{ADV-EN}$ & $\text{ADP}$ & $\text{CCE-RM}$ & $\text{CCE-DM}$\tabularnewline
\hline 
$\text{CNN}_{2}$ & 76.0 & 75.9 & 76.0 & 86.0\tabularnewline
$\text{CNN}_{3}$ & 76.7 & 76.6 & 75.7 & 87.2\tabularnewline
$\text{ResNet}_{2}$ & 85.3 & 85.3 & 84.5 & 91.0\tabularnewline
$\text{ResNet}_{3}$ & 86.1 & 86.2 & 84.9 & 91.6\tabularnewline
\hline 
\end{tabular}}
\end{table}

\paragraph{Adversarial Detectability. }

CCE-DM can distinguish between benign and adversarial examples more
easily. It is because the committee members produce a uniform prediction
for adversarial examples, while yielding a very high confident prediction
for benign examples. For example, as shown in Figure \ref{fig:pred-sp},
the committee members are highly  certain when predicting benign
examples, while they provide highly uncertain predictions with high
entropy for adversarial examples. The histogram for all images in
the test set and their adversarial examples in Figure \ref{fig:pred-hist}
demonstrate the consistency of this observation over the data distribution. 

\begin{figure}
\begin{centering}
\includegraphics[width=0.7\columnwidth]{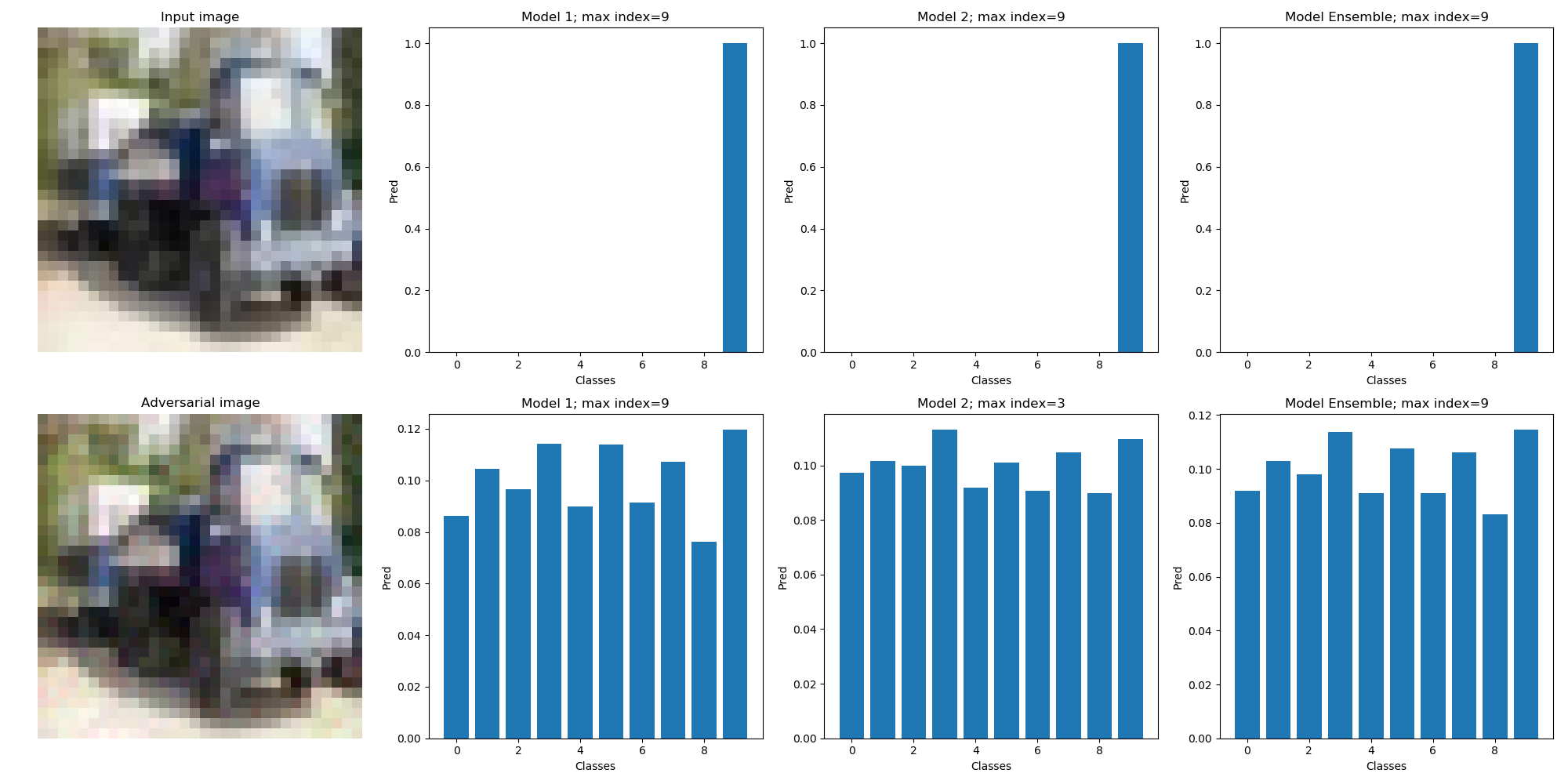}\vspace{-1.5mm}
\par\end{centering}
\caption{Prediction example in the detection mode. Top/bottom images are benign/adversarial
images. Next columns are outputs from $f^{1},f^{2},f^{en}$ \label{fig:pred-sp}}
\end{figure}

\begin{figure}
\begin{centering}
\includegraphics[width=0.6\columnwidth]{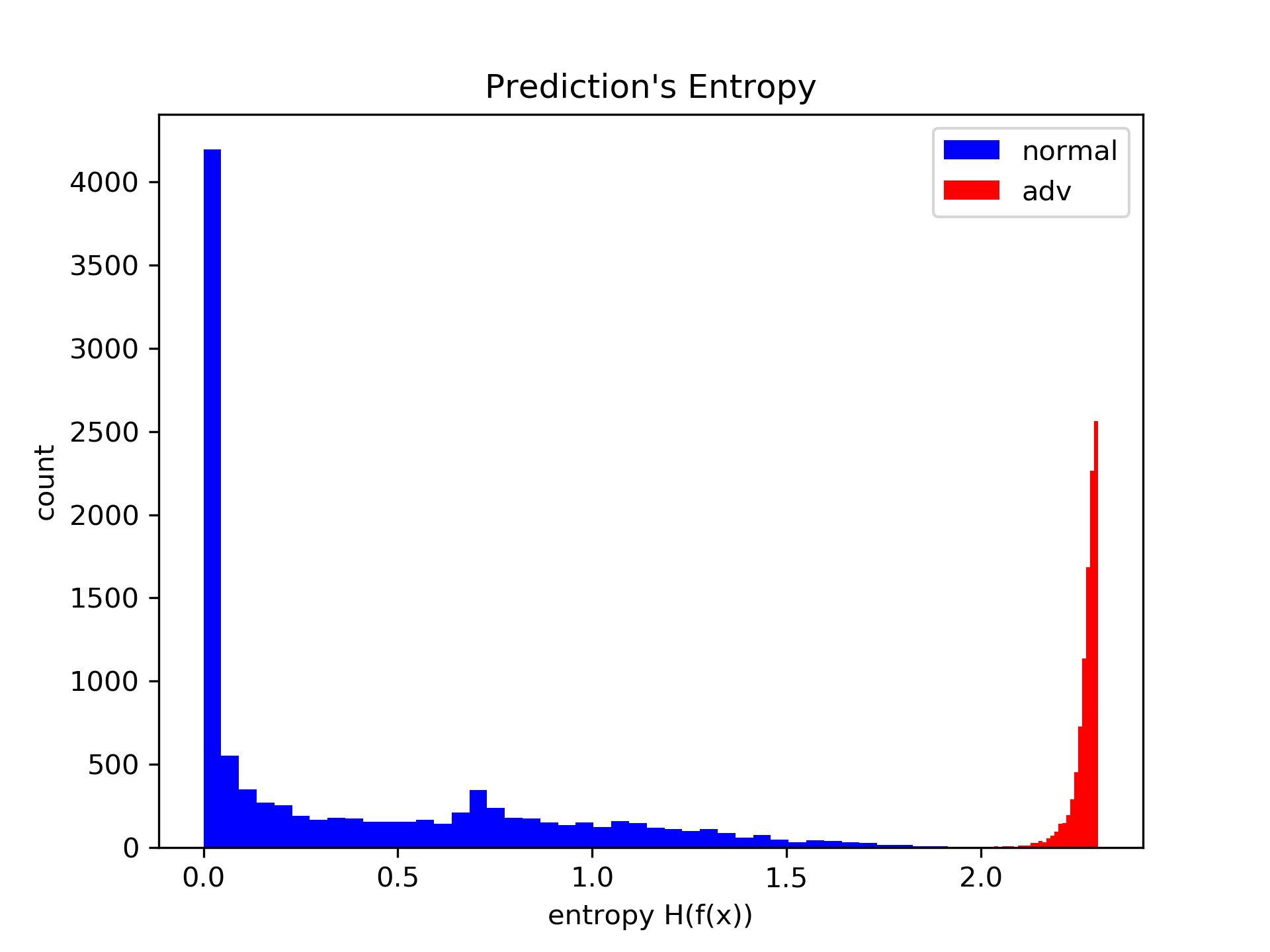}\vspace{-1.5mm}
\par\end{centering}
\caption{Histogram of prediction entropy in the detection mode\label{fig:pred-hist}}
\end{figure}

These results further inspire us to develop a simple yet effective
method to detect adversarial examples based on the entropy of the
model prediction. Following the evaluation in \citep{pang2018towards,pang2019improving},
we try with different thresholds to distinguish the benign and adversarial
examples and report the AUC score of each adversarial attack. It is
worth noting that, we do not intend to compete with other adversarial
detectors but just to show the advantage and flexibility of our CCE.
The experiment is on the CIFAR10 dataset with an ensemble of two members.
We conduct two evaluations to justify our understanding. First, we
study our detection method against three different attacks: PGD, BIM
and MIM with the same hyper-parameter setting $k=20,\epsilon=8/255,\eta=1/255$.
The result in Figure \ref{fig:det-atks} shows that our method can
accurately and consistently detect all three kind of attacks. Secondly,
we study our detection method on different attack strengths. We use
the PGD attack $k=20,\eta=1/255$ and vary the distortion bound $\epsilon$
from $1/255$ to $24/255$. The result in Figure \ref{fig:det-eps}
shows that our method can perform well on a wide range of attack strengths.
The adversary is obviously less distinguishable when decreasing its
strength. However, our method still obtains a very high AUC score
$(93.4/100)$ even under a very weak attack $(\epsilon=1/255)$, in
which adversarial images look nearly identical to the original ones. 

\begin{figure}
\centering{}\includegraphics[width=0.6\columnwidth]{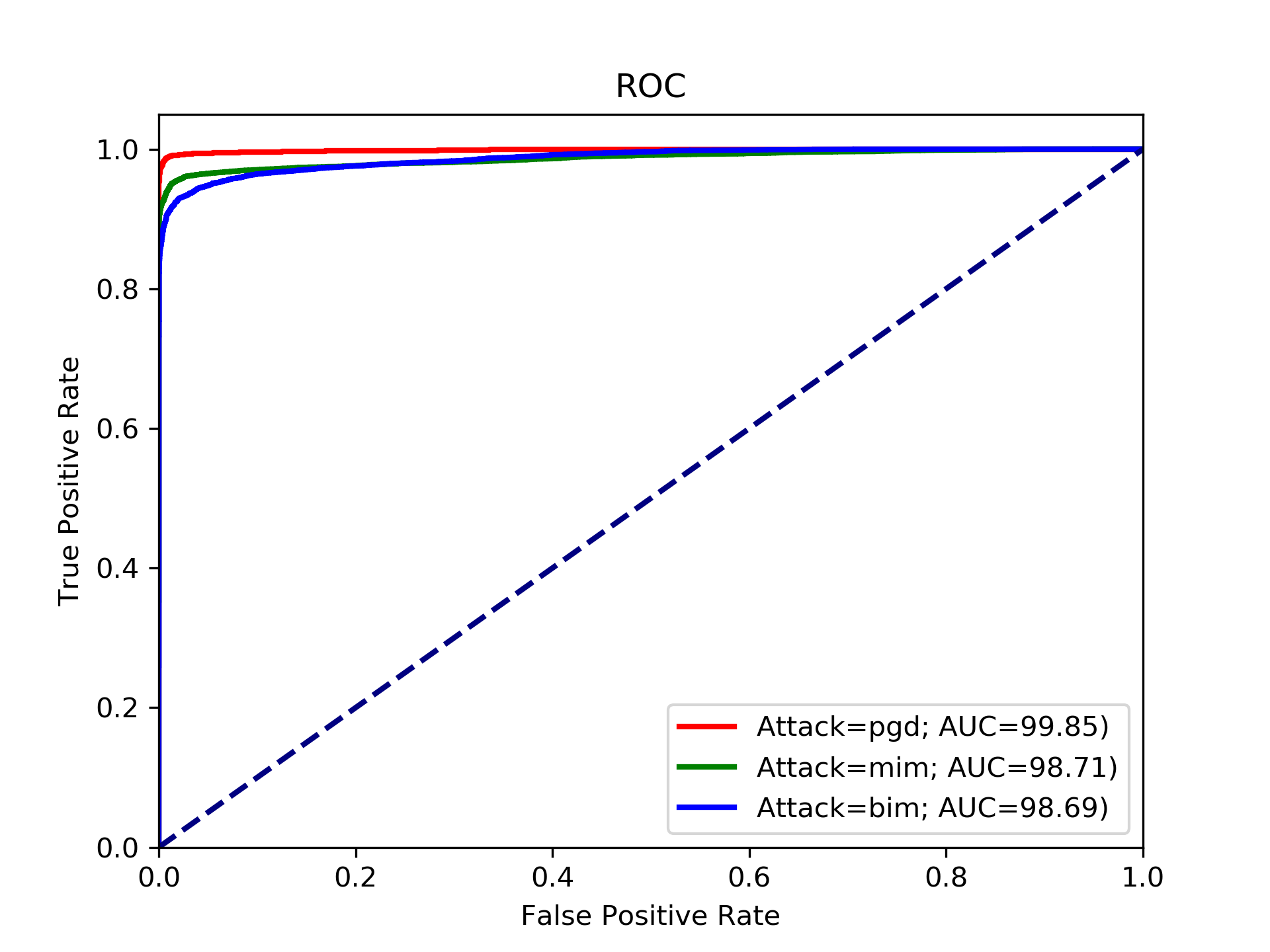}\caption{ROC of CCE-RM under multiple types of attack\label{fig:det-atks}}
\end{figure}

\begin{figure}
\centering{}\includegraphics[width=0.6\columnwidth]{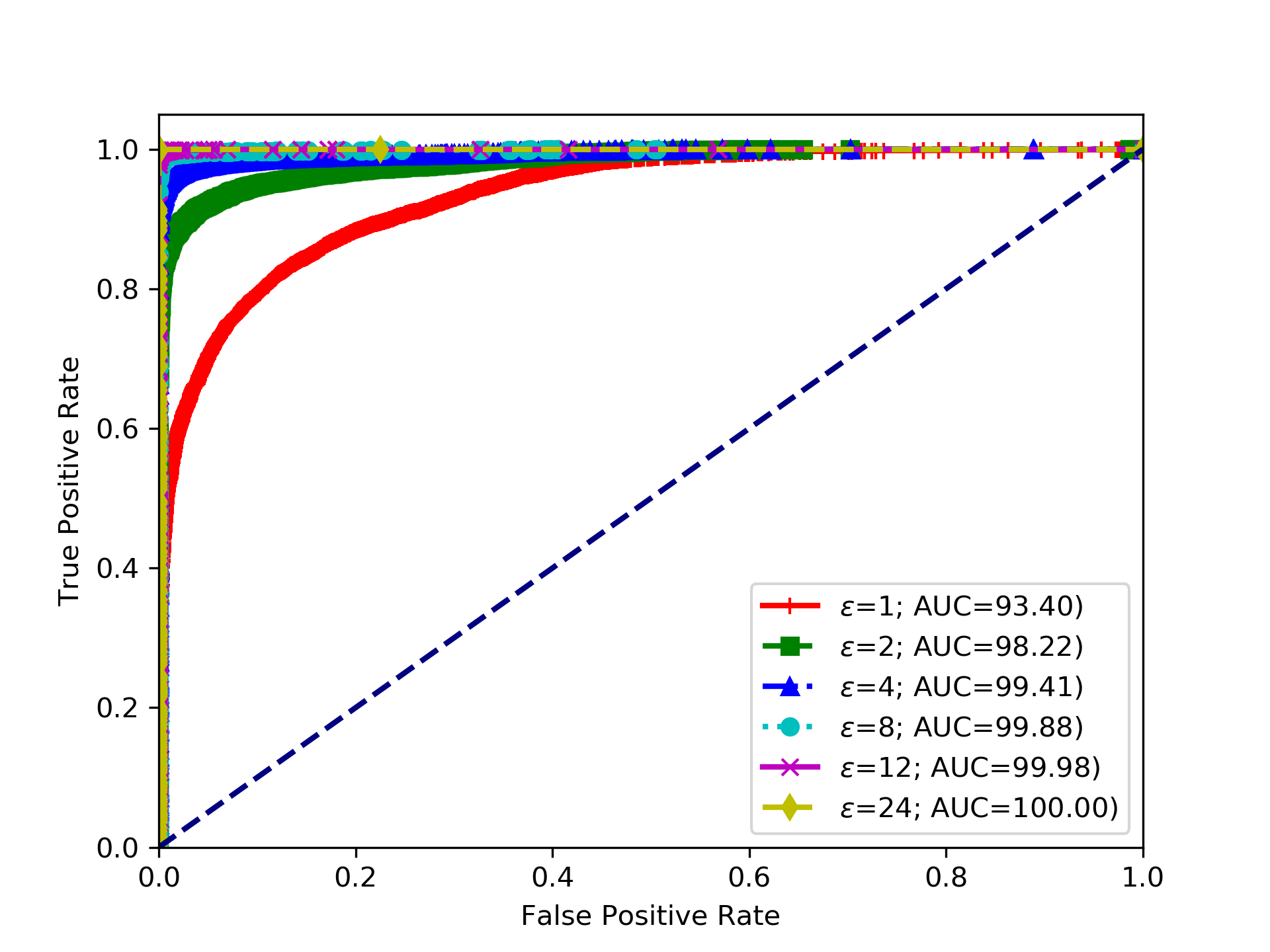}\caption{ROC of CCE-RM under multiple attack strengths\label{fig:det-eps}}
\end{figure}

\section{Conclusion}

In this paper, we explore the use of ensemble-based learning to improve
adversarial robustness. In particular, we propose a cross-collaborative
strategy by means of enforcing the transferring flow of adversarial
examples, thereby implicitly increasing the diversity of adversarial
space and improving the robustness of the ensemble. Moreover, our
proposed method can be performed in both detection and robustness
modes. We conduct extensive and comprehensive experiments to show
the improvement of our proposed method on state-of-the-art baselines.
We also provide the detailed understanding of the relationship between
the transferability and the overall robustness in the ensemble learning
context.

\paragraph{Acknowledgement.}

This work was partially supported by the Australian Defence Science
and Technology (DST) Group under the Next Generation Technology Fund
(NTGF) scheme.

\printbibliography

\pagebreak{}

\begin{center}
{\large{}Supplementary materials for ``Improving Ensemble Robustness
by Collaboratively Promoting and Demoting Adversarial Robustness''}{\large\par}
\par\end{center}

\section{Proof }
\begin{lem}
\label{lem:2}Let us define $f^{ens}\left(\cdot\right)=\frac{1}{2}f^{1}\left(\cdot\right)+\frac{1}{2}f^{2}\left(\cdot\right)$
for two given models $f^{1}$ and $f^{2}$. If $f^{1}$ and $f^{2}$
predict an example $\bx$ accurately, we have the following:

i) $\mathcal{B}_{\text{insecure}}\left(\bx,\by,f^{ens},\epsilon\right)\subset\mathcal{B}_{\text{insecure}}\left(\bx,\by,f^{1},\epsilon\right)\cup\mathcal{B}_{\text{insecure}}\left(\bx,\by,f^{2},\epsilon\right).$

ii) $\mathcal{B}_{\text{secure}}\left(\bx,\by,f^{1},\epsilon\right)\cap\mathcal{B}_{\text{secure}}\left(\bx,\by,f^{2},\epsilon\right)\subset\mathcal{B}_{\text{secure}}\left(\bx,\by,f^{ens},\epsilon\right).$
\end{lem}
\begin{proof}
It is obvious that Lemma \ref{lem:2} (i) and (ii) are equivalent.
We hence need to prove only Lemma \ref{lem:2} (ii). Consider a classification
problem on a dataset $\mathcal{D}$ with $M$ classes, the true label
of $\bx$ is $\by\in\{1,2,...,M\}$ and let $\bx'\in\mathcal{B}_{\text{secure}}\left(\bx,\by,f^{1},\epsilon\right)\cap\mathcal{B}_{\text{secure}}\left(\bx,\by,f^{2},\epsilon\right)$.
Since $f^{1}$ and $f^{2}$ predict $\bx'$ correctly with the label
$\by$, we then have:
\begin{align*}
f_{y}^{1}\left(\bx'\right) & \geq f_{j}^{1}\left(\bx'\right),\forall j\in\{1,2,...,M\},\\
f_{y}^{2}\left(\bx'\right) & \geq f_{j}^{2}\left(\bx'\right),\forall j\in\{1,2,...,M\}.
\end{align*}

This follows that
\[
f_{y}^{ens}\left(\bx'\right)\geq f_{j}^{ens}\left(\bx'\right),\forall j\in\{1,2,...,M\},
\]
which means
\[
\bx'\in\mathcal{B}_{\text{secure}}\left(\bx,\by,f^{ens},\epsilon\right).
\]
\end{proof}

\section{Related works }

In this section we introduce the most related works to our approach
including adversarial training and ensemble-based methods. 

\subsection{Adversarial Training. }

Adversarial training (ADV) can be traced back to \citep{goodfellow2014explaining},
in which a model becomes more robust by incorporating its adversarial
examples into training data. Given a model $f$, a benign example
pair $(\bx,\by)$ and an adversarial example $\bx_{a}$, the objective
function of ADV as: 
\[
\mathcal{L}_{AT}(\bx,\bx_{a},\by)=\mathcal{L}(f(\bx),\by)+\mathcal{L}(f(\bx_{a}),y)
\]

Although many defense models were broken by \citep{athalye2018obfuscated}
or gave a false sense of robustness because of the obfuscated gradient,
the adversarial training \citep{madry2017towards} was among the few
that were resilient against attacks. Many ADV's variants have been
developed including but not limited to: (1) difference in the choice
of adversarial examples, e.g., the worst-case examples \citep{goodfellow2014explaining}
or most divergent examples \citep{Zhang2019theoretically}, (2) difference
in the searching of adversarial examples, e.g., non-iterative FGSM,
Rand FGSM with random initial point or PGD with multiple iterative
gradient descent steps \citep{madry2017towards,shafahi2019adversarial}\noun{,}
(3) difference in additional regularizations, e.g., adding constraints
in the latent space \citep{zhang2019defense,bui2020improving}, (4)
difference in model architecture, e.g., activation function \citep{xie2020smooth}
or ensemble models \citep{pang2019improving}.

\subsection{Ensemble-based Defenses.}

Recent works \citep{tramer2017ensemble,kariyappa2019improving} shows
that ensemble adversarial trained models can reduce the dimensionality
of adversarial subspace \citep{tramer2017space}. There are different
approaches, however, the key ingredient of their stories is reducing
the transferability of adversarial examples between members. In \citep{tramer2017ensemble},
the authors used the crafted perturbations from static pretrained
models as augmented data to decouple the generation process of adversarial
examples of target model. However, as reported in \citep{tramer2017ensemble},
this method was designed for black-box attacks, thus still vulnerable
to white-box attacks. In \citep{kariyappa2019improving}, robustness
was achieved by aligning the gradient of committee members to be diametrically
opposed, hence reducing the shared adversarial spaces, or the transferability.
However, attempting to achieve gradient alignment is unreliable for
high-dimensional datasets and it is difficult to extend for ensemble
with more than two committee members. More recently, \citep{pang2019improving}
proposed to promote the diversity of non-maximal predictions of the
committee members (i.e., the diversity among softmax probabilities
except the highest ones) to reduce the adversarial transferability
among them. The adaptive diversity promoting (ADP) regularizer as:
$ADP_{\alpha,\beta}(\bx,\by)=\alpha\;\mathcal{H}(\mathcal{F})+\beta\;log(\mathbb{ED})$,
where $\mathcal{H}(\mathcal{F})$ is the Shannon entropy of the ensemble
prediction and $log(\mathbb{ED})$ is the logarithm of the ensemble
diversity. As reported in their paper, ADP can cooperate with adversarial
training to increase the robustness. In this case, the objective function
of ADV as: 
\begin{multline*}
\mathcal{L}_{ADP}(\bx,\bx_{a},\by)=\mathcal{L}_{AT}(\bx,\bx_{a},\by)\\
-ADP_{\alpha,\beta}(\bx,\by)-ADP_{\alpha,\beta}(\bx_{a},\by)
\end{multline*}

\section{Model architecture and training setting}

We use both standard CNN architecture and ResNet architecture in our
experiment. For ResNet architecture, we use the same architecture
and training setting as in \citep{pang2019improving}. More specifically,
we use ResNet-20 and Adam optimizer, with initialized learning rate
0.001 and reduce it by a factor 0.1 at epoch 80, 120, and 160. Table
\ref{tab:archit} summarizes the standard CNN architecture for each
ensemble member in our experiments. The architectures for the MNIST
and CIFAR10 datasets are identical with those in \citep{carlini2017towards}.
We use Adam optimization with learning rate 0.001 for all datasets.
Conv(k) represents for the Convolutional layer with k output filters
and ReLU activation. Kernel size 3 and stride 1 for every convolution
layer. FC(k) represents for the Fully Connected layers with k output
filters without ReLU activation. Dropout rate is 0.5. We train models
in 180 epochs for both CIFAR10 and CIFAR100 datasets and in 100 epochs
for the MNIST dataset. 

\begin{table}
\caption{Model architectures for experimental section\label{tab:archit}}

\centering{}\resizebox{0.4\textwidth}{!}{\centering\setlength{\tabcolsep}{2pt}
\begin{tabular}{ccc}
\hline 
\textbf{MNIST} & \textbf{CIFAR10} & \textbf{CIFAR100}\tabularnewline
\hline 
2 x Conv(32) & 2 x Conv(64) & 3 x Conv(64)\tabularnewline
MaxPool & MaxPool & MaxPool\tabularnewline
2 x Conv(64) & 2 x Conv(128) & 3 x Conv(128)\tabularnewline
MaxPool & MaxPool & MaxPool\tabularnewline
FC(200), ReLU & FC(256), ReLU & FC(256), ReLU\tabularnewline
Dropout(0.5) & Dropout(0.5) & Dropout(0.5)\tabularnewline
FC(200), ReLU & FC(256), ReLU & 2 x (FC(256), ReLU)\tabularnewline
FC(10) & FC(10) & FC(100)\tabularnewline
Softmax & Softmax & Softmax\tabularnewline
\hline 
\end{tabular}}
\end{table}

\paragraph{Comparison on the training time. }

Our method requires to find the adversarial examples of each member
and do cross inference, therefore, it takes a longer training process.
We measured the training time (per epoch) on our machine with Nvidia
RTX Titan GPU, using ResNet architecture (N=2,3) with batch size 64
on the CIFAR10 dataset and summarize as in Table \ref{tab:train-time}. 

\begin{table}
\caption{Comparison on the training time on the CIFAR10 dataset using ResNet
architecture\label{tab:train-time}}

\centering{}%
\begin{tabular}{ccc}
\hline 
Model & N=2 & N=3\tabularnewline
\hline 
ADV (N=1) & 109s & 109s\tabularnewline
ADV-EN & 205s & 319s\tabularnewline
ADP & 210s & 328s\tabularnewline
Ours & 356s & 546s\tabularnewline
\hline 
\end{tabular}
\end{table}

\section{White-box attacks evaluation }

In addition to the result in the experimental section, we provide
further results on the evaluation of adversarial robustness under
white-box attacks. Firstly, we explain in detail the metrics of interest
in our experiments. Secondly, we provide an ablation study to show
the impact of the transferring flow to the improvement. 

\subsection{Robustness evaluation metrics }

\subsubsection{Static attack and Adaptive attack. }

There are two scenarios of attacks on an ensemble model \citep{he2017adversarial}.
The first scenario is \emph{static attack}, in which the attacker
is not aware of the ensemble method (i.e., how to do the ensemble
for making the final prediction). The other scenario is \emph{adaptive attack},
where the attacker has full access to the ensemble method and adapts
attacks accordingly. In our experiments, we make use of the adaptive
attack, which is a considerably stronger attack. 

\subsubsection{Non-targeted attack and Multiple-targeted attack.}

We use both non-targeted attack ($\mathcal{A}$) and multiple-targeted
attack ($mul\mathcal{A}$) in our evaluation. The non-targeted attack
obtains adversarial examples by maximizing the loss w.r.t its true
label, resulting in any non-true label prediction. The multiple-targeted
attack is undertaken by performing simultaneously targeted attacks
for all possible data labels (10 for CIFAR10 and 100 for CIFAR100)
and being counted if any individual targeted-attack is successful.
While the non-targeted attack considers only one direction of the
gradient, the multiple-targeted attack takes many directions into
account, therefore, being considered as a much stronger attack.

\subsection{Ablation study }

We provide an ablation study to compare CCE-RM with CCE-Base (which
disables promoting and demoting operations by setting $\lambda_{pm}=\lambda_{dm}=0$).
Firstly, the comparison in Table \ref{tab:ablation-base} shows that
even CCE-Base variant can beat ADP method on both CIFAR10 and CIFAR100
datasets. This surpassness can be explained from the fact our proposed
method encourages the diversity of its committee members. More specifically,
each member is reinforced with two data sources: clean data $\{\bx\}$
and adversarial examples $\{\bx_{a}^{n}\}$, which becomes more diverge
due to the gradually more divergence of the committee models and the
random initialization of PGD at the step $0$. From this point of
view, our method can be linked to the bagging technique in traditional
ensemble learning, which is a well-known method to produce the diversity
in the ensemble. Secondly, CCE-RM shows a huge improvement over CCE-Base
in both CIFAR10 and CIFAR100 datasets. This result demonstrates the
impact of the transferring flow, which offers better collaboration
among members. 

\begin{table}
\caption{Ablation study on the impact of the transferring flow. Note that $mul\mathcal{A}$
represents for the multiple-targeted attack by adversary $\mathcal{A}$.\label{tab:ablation-base}}

\begin{centering}
\subfloat[Evaluation on CIFAR10 dataset. We commonly use $\epsilon=8/255,\eta=2/255$]{
\begin{centering}
\resizebox{0.4\textwidth}{!}{\centering\setlength{\tabcolsep}{2pt}
\begin{tabular}{cccc>{\centering}p{0.1in}ccc}
\hline 
 & $\text{ADP}_{2}$ & $\text{CCE-Base}_{2}$ & $\text{CCE-RM}_{2}$ &  & $\text{ADP}_{3}$ & $\text{CCE-Base}_{3}$ & $\text{CCE-RM}_{3}$\tabularnewline
\hline 
Non-att (Nat. acc.) & 75.9 & 75.8 & \textbf{76.0} &  & 76.6 & 76.4 & 75.7\tabularnewline
PGD $k=100$ & 42.2 & 43.4 & \textbf{44.7} &  & 43.9 & 44.5 & \textbf{46.8}\tabularnewline
BIM $k=100$ & 42.2 & 43.4 & \textbf{44.9} &  & 43.8 & 44.5 & \textbf{46.8}\tabularnewline
MIM $k=100$ & 42.4 & 44.1 & \textbf{45.4} &  & 44.2 & 45.0 & \textbf{47.2}\tabularnewline
mul-PGD $k=20$ & 27.8 & 31.1 & \textbf{31.9} &  & 32.4 & 32.4 & \textbf{36.9}\tabularnewline
mul-BIM $k=20$ & 27.2 & 30.8 & \textbf{31.6} &  & 29.8 & 32.2 & \textbf{34.1}\tabularnewline
mul-MIM $k=20$ & 28.3 & 31.5 & \textbf{32.3} &  & 30.7 & 32.9 & \textbf{34.6}\tabularnewline
SPSA & 41.5 & 44.3 & \textbf{45.2} &  & 46.1 & 47.2 & \textbf{47.5}\tabularnewline
Auto-Attack & 24.4 & 29.2 & \textbf{29.9} &  & 28.1 & 31.5 & \textbf{31.9}\tabularnewline
\hline 
\end{tabular}}
\par\end{centering}
}
\par\end{centering}
\centering{}\subfloat[Evaluation on CIFAR100 dataset. We commonly use $\epsilon=0.01,\eta=0.001$]{
\begin{centering}
\resizebox{0.4\textwidth}{!}{\centering\setlength{\tabcolsep}{2pt}
\begin{tabular}{cccc>{\centering}p{0.1in}ccc}
\hline 
 & $\text{ADP}_{2}$ & $\text{CCE-Base}_{2}$ & $\text{CCE-RM}_{2}$ &  & $\text{ADP}_{3}$ & $\text{CCE-Base}_{3}$ & $\text{CCE-RM}_{3}$\tabularnewline
\hline 
Non-att (Nat. acc.) & 48.0 & 51.1 & \textbf{53.4} &  & 52.6 & 54.2 & \textbf{54.4}\tabularnewline
PGD $k=100$ & 30.9 & 33.6 & \textbf{35.3} &  & 36.2 & 37.0 & \textbf{39.5}\tabularnewline
BIM $k=100$ & 31.0 & 33.7 & \textbf{35.2} &  & 36.2 & 37.1 & \textbf{39.4}\tabularnewline
MIM $k=100$ & 30.8 & 33.5 & \textbf{35.3} &  & 36.1 & 37.2 & \textbf{39.6}\tabularnewline
mul-PGD $k=20$ & 20.1 & 23.0 & \textbf{24.2} &  & 24.8 & 26.2 & \textbf{28.4}\tabularnewline
mul-BIM $k=20$ & 19.4 & 22.6 & \textbf{23.7} &  & 24.5 & 25.9 & \textbf{28.1}\tabularnewline
mul-MIM $k=20$ & 20.3 & 23.1 & \textbf{24.1} &  & 25.1 & 26.4 & \textbf{28.6}\tabularnewline
SPSA & 24.1 & \textbf{32.1} & 31.8 &  & 32.5 & \textbf{35.1} & 35.0\tabularnewline
Auto-Attack & 14.8 & \textbf{22.0} & 21.9 &  & 23.0 & \textbf{26.1} & 25.9\tabularnewline
\hline 
\end{tabular}}
\par\end{centering}
}
\end{table}

In addition, we study the impact of each PO and DO to the final performance
by evaluating them separately. Table \ref{tab:ablation-separate}
shows the comparison when disabling one of these operations while
varying the other. It can be observed that: (i) the ensemble tends
to be detection mode (i.e., increasing natural performance and adversarial
detectability while sacrificing its robustness) when increasing DO's
strength ($\lambda_{dm}\geq2$), (ii) the ensemble tends to reduce
its robustness slightly when increasing PO's strength, (iii) neither
PO nor DO can improve the robustness alone, which shows the important
of the transferring flow. These observations are inline with the properties
of PO and DO which have been mentioned in the main paper. The parameter
$\lambda_{pm}(\lambda_{dm})$ controls the level of the agreement
(disagreement) of models $\{f^{i}\},i\in[1,N]$ and model $f^{j},j\neq i$
on the same adversarial example $x_{a}^{j}$. Therefore the observation
(i) can be explained by the fact that by disabling cPO ($\lambda_{pm}=0$)
and strengthening DO, our method encourages the disagreement among
members on the same data example, therefore, increases the negative
correlation among them. In contrast, by disabling DO and increasing
cPO's strength, our method increases the agreement among members,
therefore, increases the positive correlation. The increasing of the
positive correlation among members reduces the diversity of adversarial
space, therefore, explains the observation (ii).

\begin{table}
\caption{Ablation study on the impact of each operation PO/DO. We commonly
use $\epsilon=8/255,\eta=2/255$. Note that $mul\mathcal{A}$ represents
for the multiple-targeted attack by adversary $\mathcal{A}$.\label{tab:ablation-separate}}

\begin{centering}
\subfloat[Using DO only by disabling cPO ($\lambda_{pm}=0$)]{
\begin{centering}
\resizebox{0.4\textwidth}{!}{\centering\setlength{\tabcolsep}{2pt}
\begin{tabular}{ccccc}
\hline 
 & $\lambda_{dm}=0$ & $\lambda_{dm}=1$ & $\lambda_{dm}=2$ & $\lambda_{dm}=5$\tabularnewline
\hline 
Non-att (Nat. acc.) & 75.8 & 76.6 & 83.2 & 86.0\tabularnewline
PGD $k=100$ & 43.4 & 43.3 & 23.9 & 26.1\tabularnewline
BIM $k=100$ & 43.4 & 43.2 & 24.0 & 26.2\tabularnewline
MIM $k=100$ & 44.1 & 43.6 & 31.1 & 21.7\tabularnewline
mul-PGD $k=20$ & 31.1 & 29.9 & 14.9 & 20.1\tabularnewline
mul-BIM $k=20$ & 30.8 & 29.8 & 13.1 & 19.8\tabularnewline
mul-MIM $k=20$ & 31.5 & 30.3 & 28.6 & 21.7\tabularnewline
SPSA & 44.3 & 44.4 & 30.4 & 5.3\tabularnewline
Auto-Attack & 29.2 & 29.8 & 1.6 & 0.2\tabularnewline
\hline 
\end{tabular}}
\par\end{centering}
}
\par\end{centering}
\centering{}\subfloat[Using PO only by disabling DO ($\lambda_{dm}=0$)]{
\centering{}\resizebox{0.4\textwidth}{!}{\centering\setlength{\tabcolsep}{2pt}
\begin{tabular}{ccccc}
\hline 
 & $\lambda_{pm}=0$ & $\lambda_{pm}=1$ & $\lambda_{pm}=2$ & $\lambda_{pm}=5$\tabularnewline
\hline 
Non-att (Nat. acc.) & 75.8 & 76.4 & 76.2 & 77.1\tabularnewline
PGD $k=100$ & 43.4 & 42.7 & 42.7 & 41.3\tabularnewline
BIM $k=100$ & 43.4 & 42.8 & 42.7 & 41.5\tabularnewline
MIM $k=100$ & 44.1 & 43.0 & 43.2 & 41.9\tabularnewline
mul-PGD $k=20$ & 31.1 & 30.4 & 30.0 & 29.7\tabularnewline
mul-BIM $k=20$ & 30.8 & 30.3 & 29.9 & 29.6\tabularnewline
mul-MIM $k=20$ & 31.5 & 30.8 & 30.4 & 30.1\tabularnewline
SPSA & 44.3 & 44.3 & 45.5 & 46.2\tabularnewline
Auto-Attack & 29.2 & 29.7 & 28.8 & 30.0\tabularnewline
\hline 
\end{tabular}}}
\end{table}

\section{Black-box attacks evaluation}

We investigate the transferability of adversarial examples among models
and evaluate the robustness under black-box attacks. The experiment
is conducted on the CIFAR10 and CIFAR100 datasets, with ensemble of
two members. We use PGD to challenge each ensemble model to generate
adversarial examples then transfer these adversarial examples to attack
other models. The PGD configuration for the CIFAR10 dataset is $k=20,\eta=2/255,\epsilon\in\{8/255,12/255\}$,
while that for the CIFAR100 dataset is $k=20,\eta=0.001,\epsilon\in\{0.01,0.02\}$.
The result is shown in Figure \ref{fig:trans-bb}. The element $a^{(i,j)}$
in each sub-table represents the robust accuracy when adversarial
examples from model $i$ attack model $j$. $\emph{{NAT}}$ represents
for the natural model which does not engage with any defense method.
It is worth noting that, the diagonal in each sub-table represents
robust accuracies against the white-box attacks, which has been discussed
in the section above. 

Firstly, the first row in each sub-table shows the robust accuracy
against adversarial examples which are transferred from the natural
model ($\emph{{NAT}}$). This result shows that our method outperforms
baseline methods on the CIFAR100 dataset, but to be weaker than other
on the CIFAR10 dataset. Secondly, each column in each sub-table compares
the attack strength of different models on the same defense model.
The comparison on these columns shows that adversarial examples which
are crafted from our CCE-RM attack better than those crafted from
other methods (i.e., by giving lower robust accuracy). This result
indicates that our method generates stronger adversarial examples
than other methods. 

\begin{table}
\caption{Blackbox attack evaluation on CIFAR10 dataset}

\begin{centering}
\subfloat[PGD attack with $k=20,\epsilon=8/255,\eta=2/255$]{
\centering{}\resizebox{0.4\textwidth}{!}{\centering\setlength{\tabcolsep}{2pt}
\begin{tabular}{|c|>{\centering}p{1.5cm}|>{\centering}p{1.5cm}|>{\centering}p{1.5cm}|>{\centering}p{1.5cm}|}
\hline 
 & NAT & ADV-EN & ADP & CCE-RM\tabularnewline
\hline 
NAT & 9.2 & 76.2 & 75.6 & 74.5\tabularnewline
\hline 
ADV-EN & 59.5 & 40.7 & 58.5 & 58.4\tabularnewline
\hline 
ADP & 58.4 & 61.8 & 42.9 & 59.9\tabularnewline
\hline 
CCE-RM & 57.3 & 58.9 & 57.4 & 45.5\tabularnewline
\hline 
\end{tabular}}}
\par\end{centering}
\centering{}\subfloat[PGD attack with $k=20,\epsilon=12/255,\eta=2/255$]{
\centering{}\resizebox{0.4\textwidth}{!}{\centering\setlength{\tabcolsep}{2pt}
\begin{tabular}{|c|>{\centering}p{1.5cm}|>{\centering}p{1.5cm}|>{\centering}p{1.5cm}|>{\centering}p{1.5cm}|}
\hline 
 & NAT & ADV-EN & ADP & CCE-RM\tabularnewline
\hline 
NAT & 8.0 & 75.5 & 75.1 & 74.0\tabularnewline
\hline 
ADV-EN & 40.1 & 23.6 & 44.5 & 45.2\tabularnewline
\hline 
ADP & 38.4 & 48.5 & 24.9 & 47.1\tabularnewline
\hline 
CCE-RM & 35.9 & 42.9 & 41.3 & 27.8\tabularnewline
\hline 
\end{tabular}}}
\end{table}

\begin{table}
\caption{Blackbox attack evaluation on CIFAR100 dataset}

\begin{centering}
\subfloat[PGD attack with $k=20,\epsilon=0.01,\eta=0.001$]{
\centering{}\resizebox{0.4\textwidth}{!}{\centering\setlength{\tabcolsep}{2pt}
\begin{tabular}{|c|>{\centering}p{1.6cm}|>{\centering}p{1.6cm}|>{\centering}p{1.6cm}|>{\centering}p{1.6cm}|}
\hline 
 & NAT & ADV-EN & ADP & CCE-RM\tabularnewline
\hline 
NAT & 14.7 & 43.6 & 49.3 & 52.9\tabularnewline
\hline 
ADV-EN & 42.9 & 31.1 & 48.1 & 51.6\tabularnewline
\hline 
ADP & 41.8 & 41.6 & 32.8 & 50.9\tabularnewline
\hline 
CCE-RM & 39.4 & 40.2 & 45.8 & 36.1\tabularnewline
\hline 
\end{tabular}}}
\par\end{centering}
\centering{}\subfloat[PGD attack with $k=20,\epsilon=0.02,\eta=0.001$]{
\centering{}\resizebox{0.4\textwidth}{!}{\centering\setlength{\tabcolsep}{2pt}
\begin{tabular}{|c|>{\centering}p{1.6cm}|>{\centering}p{1.6cm}|>{\centering}p{1.6cm}|>{\centering}p{1.6cm}|}
\hline 
 & NAT & ADV-EN & ADP & CCE-RM\tabularnewline
\hline 
NAT & 13.0 & 43.6 & 49.2 & 52.9\tabularnewline
\hline 
ADV-EN & 39.9 & 24.3 & 46.9 & 50.4\tabularnewline
\hline 
ADP & 38.3 & 40.6 & 25.2 & 49.4\tabularnewline
\hline 
CCE-RM & 34.5 & 38.2 & 43.0 & 27.5\tabularnewline
\hline 
\end{tabular}}}
\end{table}

\section{Loss surface visualization }

In addition to the quantitative evaluation on the adversarial robustness,
we would like to provide two additional visualizations which further
demonstrate our improvement. The visualizations are conducted on the
same image from CIFAR10 dataset  with the ensemble of two models.
First, we visualize the prediction probability of each ensemble member
and the entire ensemble with the same two types of input which are
a benign example and adversarial example of the benign one. The visualization
as Figure \ref{fig:visual_pred_prob} shows that our method can produce
a high confident prediction unlike ADP which has a less confident
prediction because of its diversity promoting method. Secondly, we
visualize the loss surface around the adversarial example $\bx_{a}$
w.r.t three different of model:$f^{1},f^{2}$ and $f^{en}$. We generate
a grid of neighborhood images $\{\bx_{a}+i*u+j*v\}$ where $u=\nabla_{\bx}\mathcal{C}\left(f(\bx_{a}),\by\right)$
is the gradient of the prediction loss w.r.t the input and $v$ is
the random perpendicular vector to $u$. In each sub-figure, the left
image is the adversarial example of interest while the middle and
the right image depict the loss surface and the predicted labels corresponding
with the neighbor grid. Our method can produce correct labels in entire
the neighborhood region, unlike other methods that still have an incorrect
prediction region. Therefore, our method can produce a smoother surface
around the adversarial example which further explains the better robustness
in our method. 

\begin{figure}
\begin{centering}
\subfloat[ADV-EN. ]{\begin{centering}
\includegraphics[width=0.9\columnwidth]{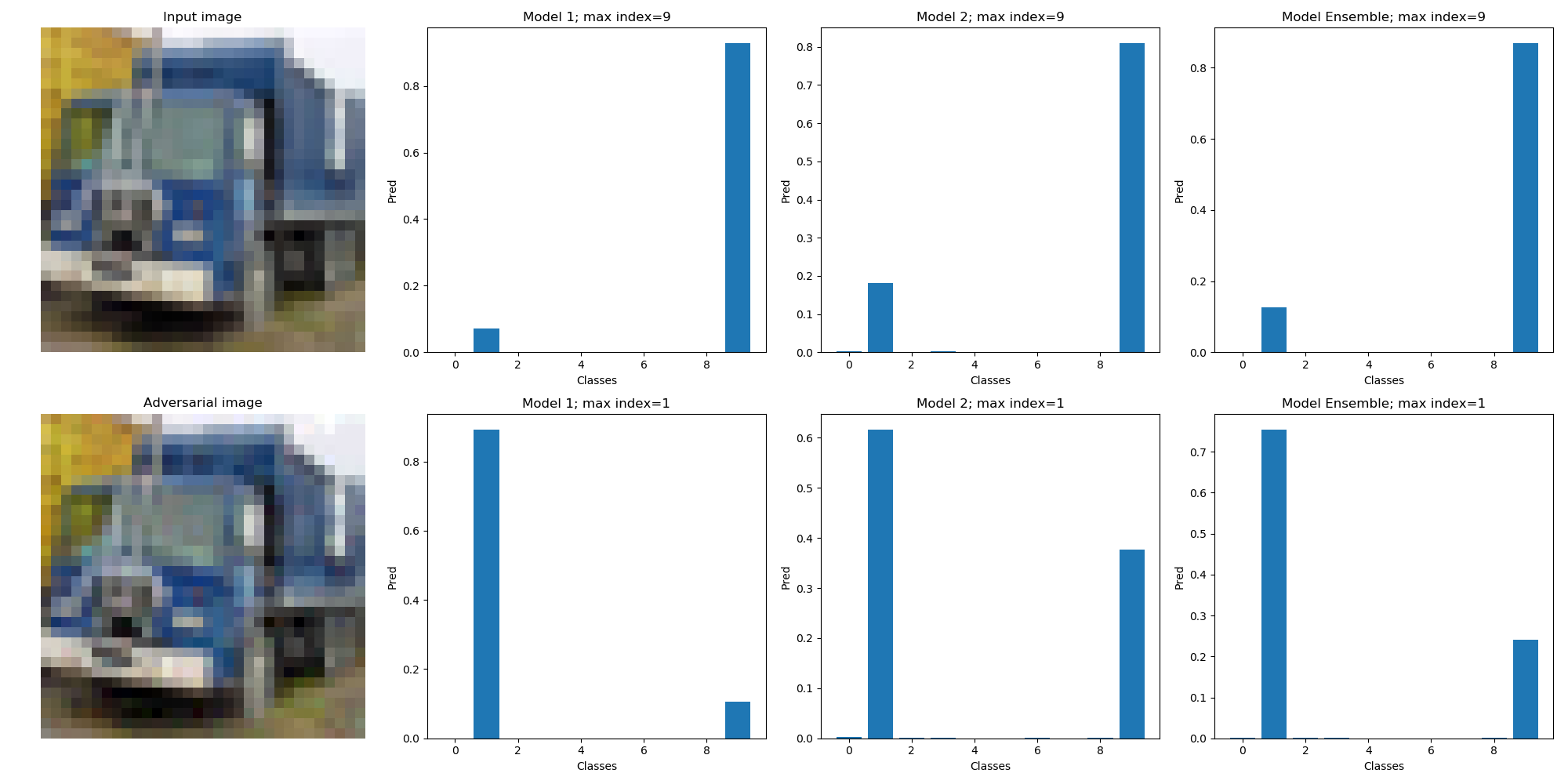}
\par\end{centering}

}\medskip{}
\subfloat[ADP.]{\begin{centering}
\includegraphics[width=0.9\columnwidth]{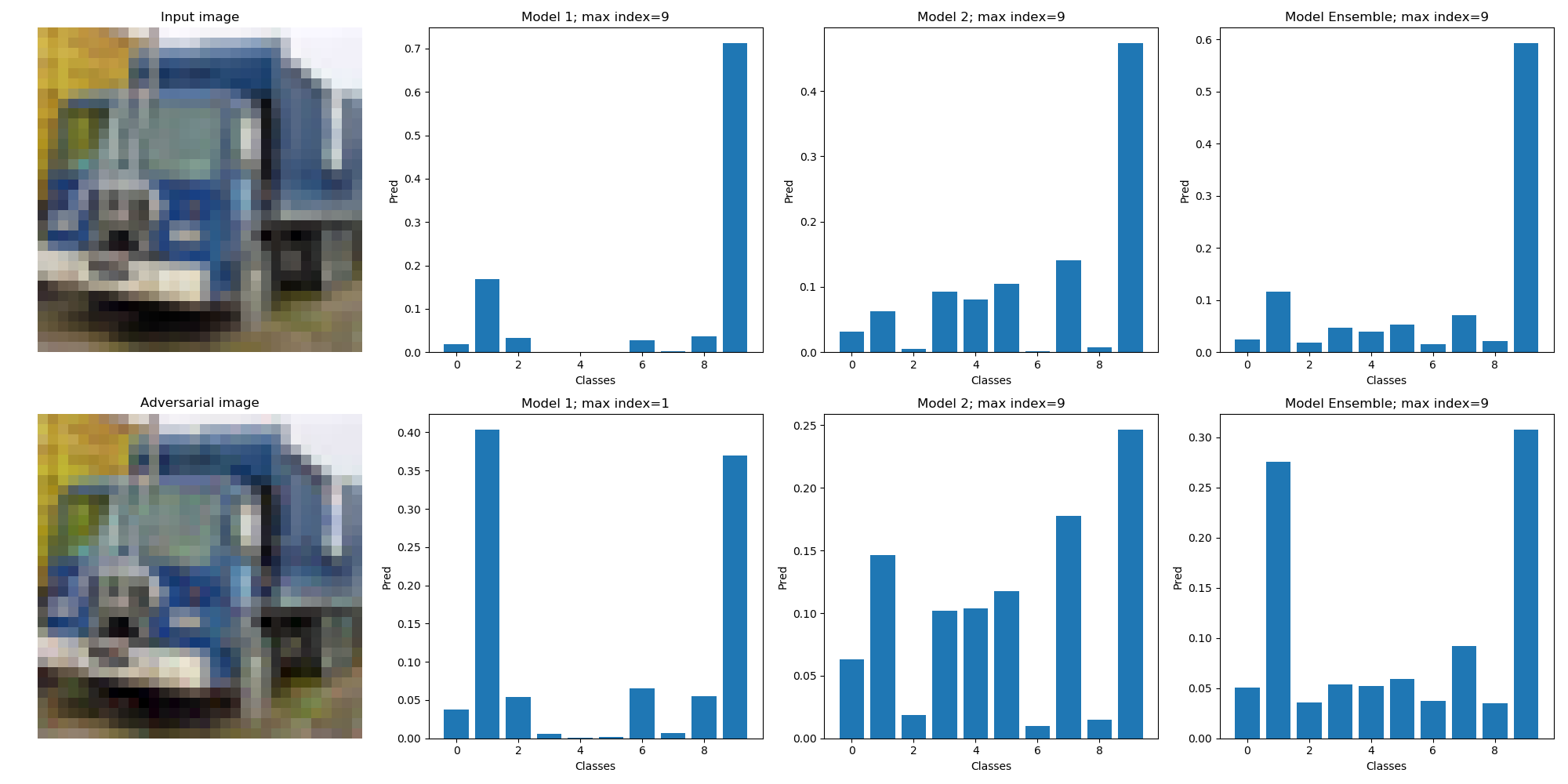}
\par\end{centering}
}\medskip{}
\subfloat[CCE-RM.]{\begin{centering}
\includegraphics[width=0.9\columnwidth]{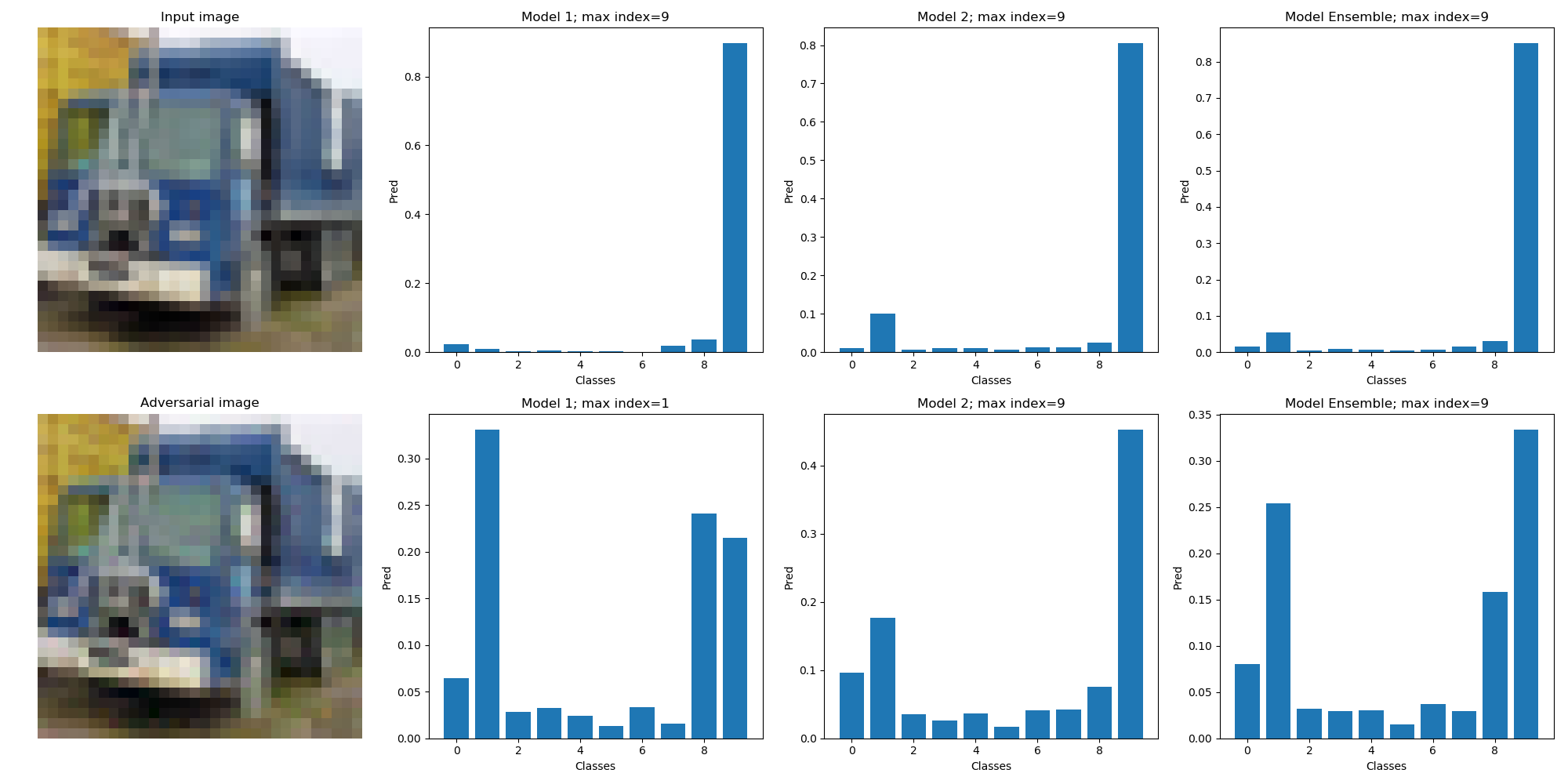}
\par\end{centering}

}
\par\end{centering}
\caption{Prediction example. Top/bottom images are benign/adversarial images.
Next columns are outputs from $f^{1},f^{2},f^{en}$.\label{fig:visual_pred_prob}}
\end{figure}

\begin{figure}
\begin{centering}
\subfloat[Prediction surface of model $f^{1}$.]{\begin{centering}
\includegraphics[width=1\columnwidth]{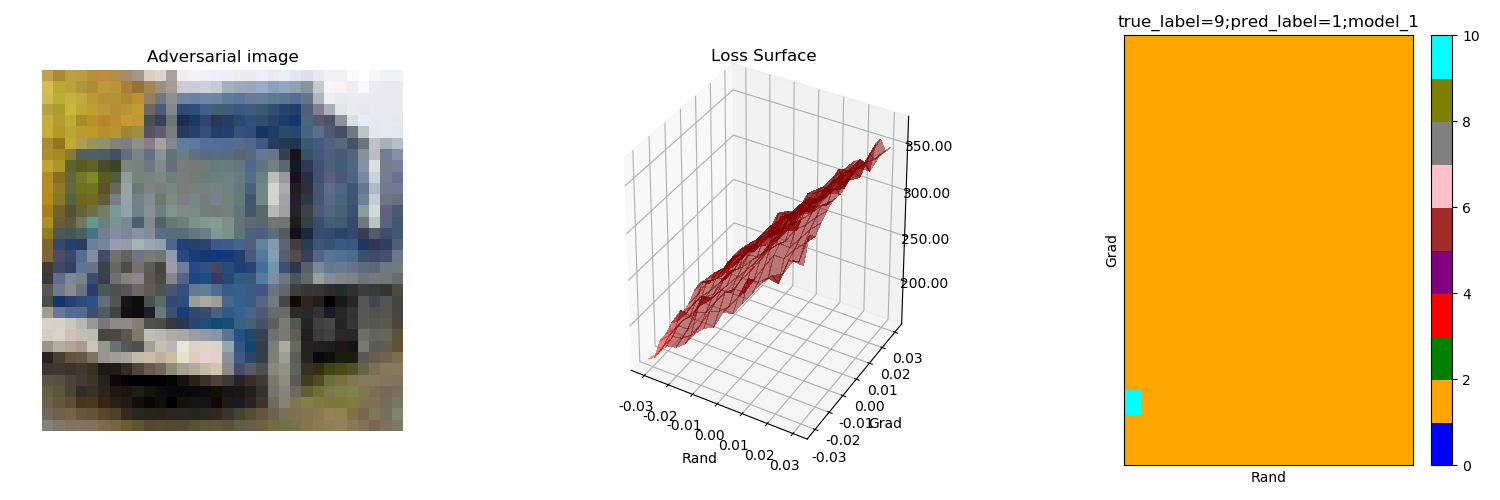}
\par\end{centering}
}\medskip{}
\subfloat[Prediction surface of model $f^{2}$.]{\begin{centering}
\includegraphics[width=1\columnwidth]{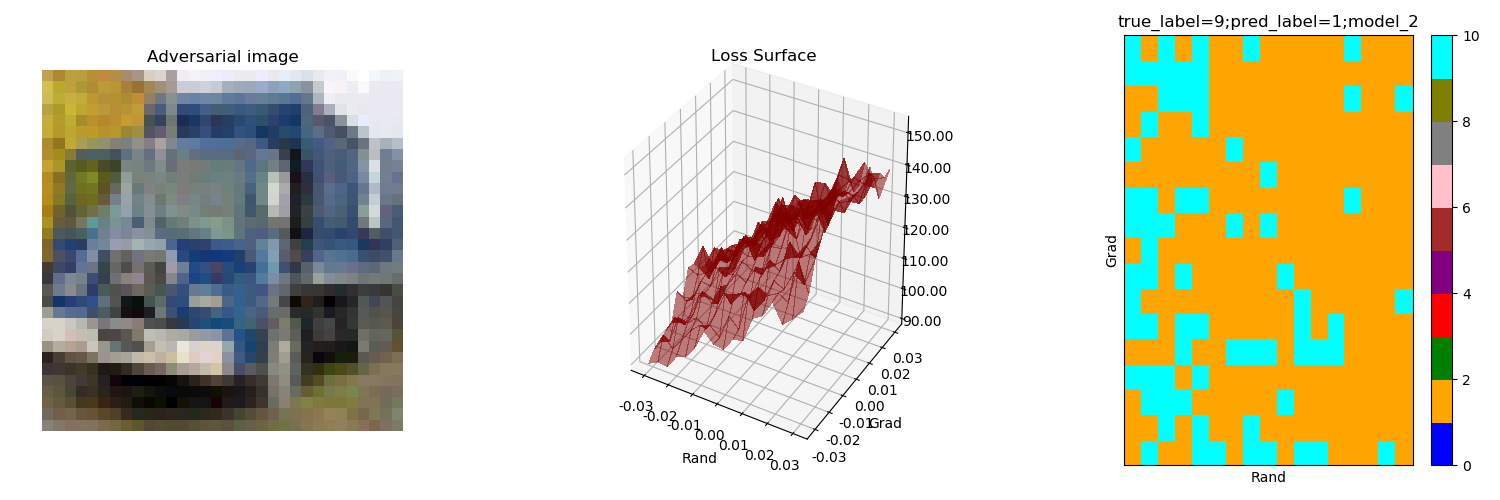}
\par\end{centering}
}\medskip{}
\subfloat[Prediction surface of model $f^{en}$.]{\begin{centering}
\includegraphics[width=1\columnwidth]{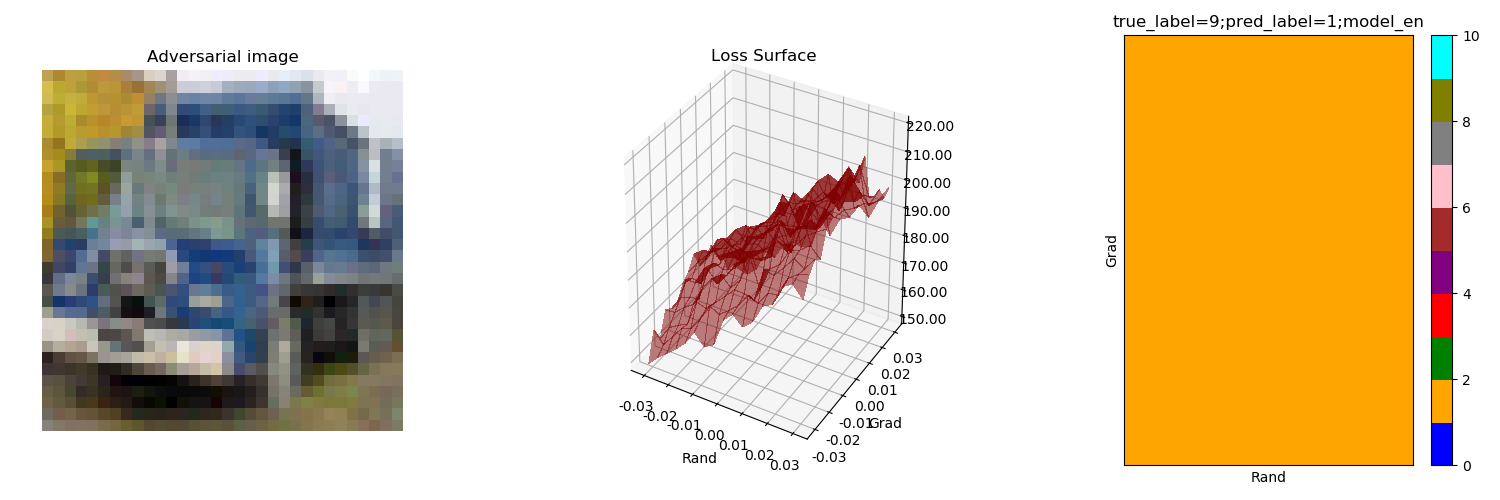}
\par\end{centering}
}
\par\end{centering}
\caption{Loss surface around adversarial example of ADV-EN method. Left: Adversarial
input. Middle: Loss surface. Right: Predicted labels.\label{fig:surface_adv}}
\end{figure}

\begin{figure}
\begin{centering}
\subfloat[Prediction surface of model $f^{1}$.]{\begin{centering}
\includegraphics[width=1\columnwidth]{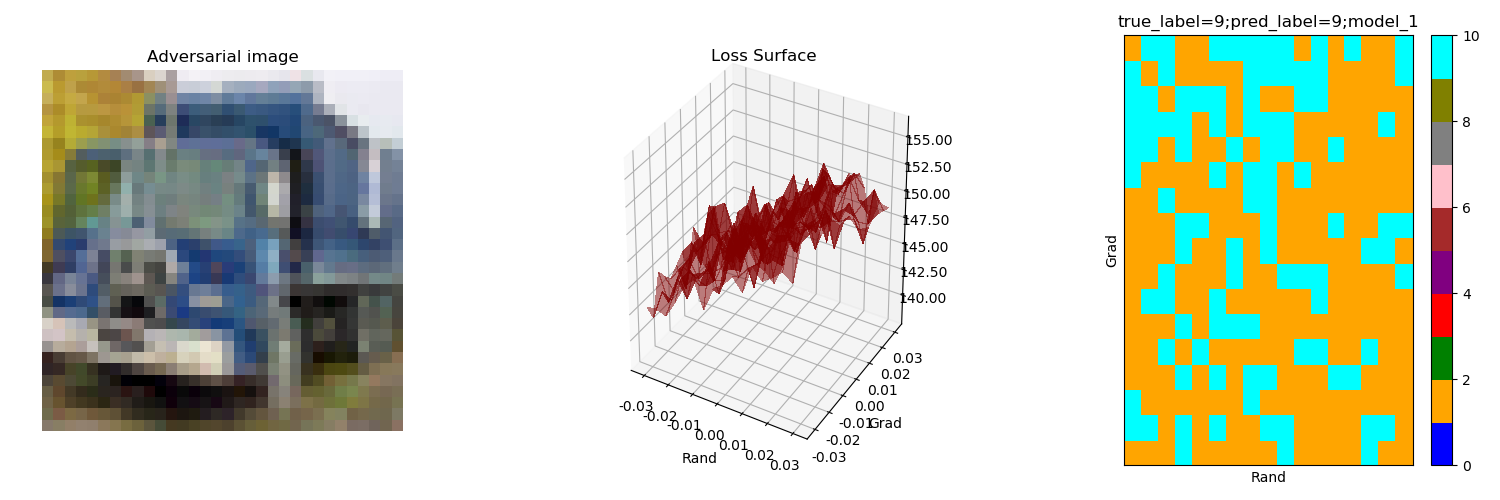}
\par\end{centering}
}\medskip{}
\subfloat[Prediction surface of model $f^{2}$.]{\begin{centering}
\includegraphics[width=1\columnwidth]{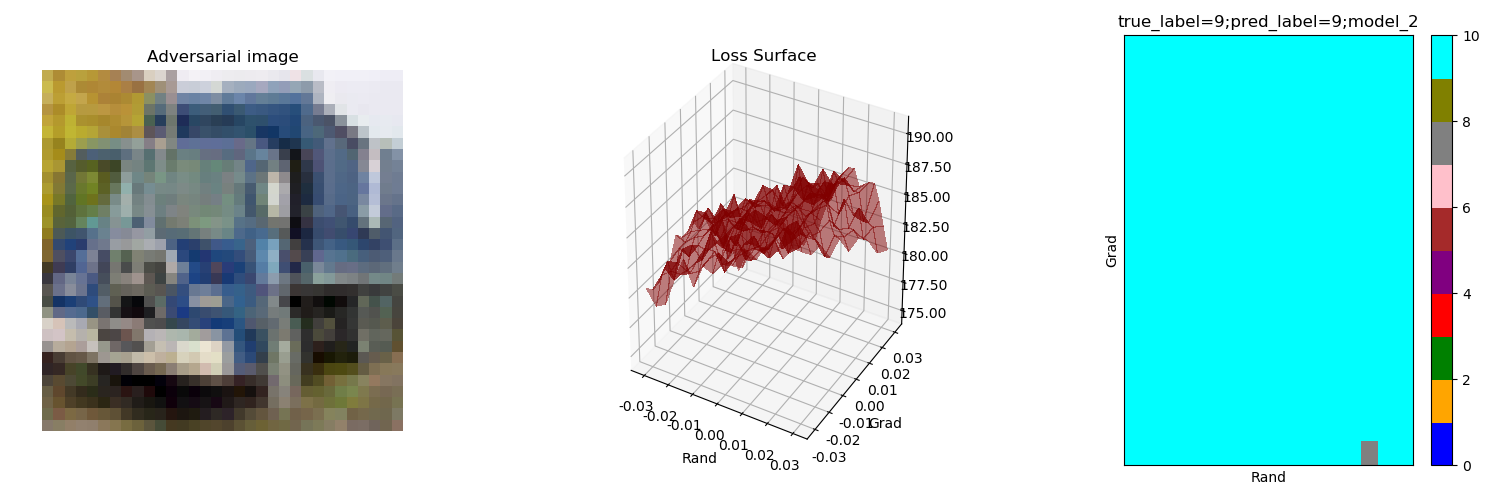}
\par\end{centering}
}\medskip{}
\subfloat[Prediction surface of model $f^{en}$.]{\begin{centering}
\includegraphics[width=1\columnwidth]{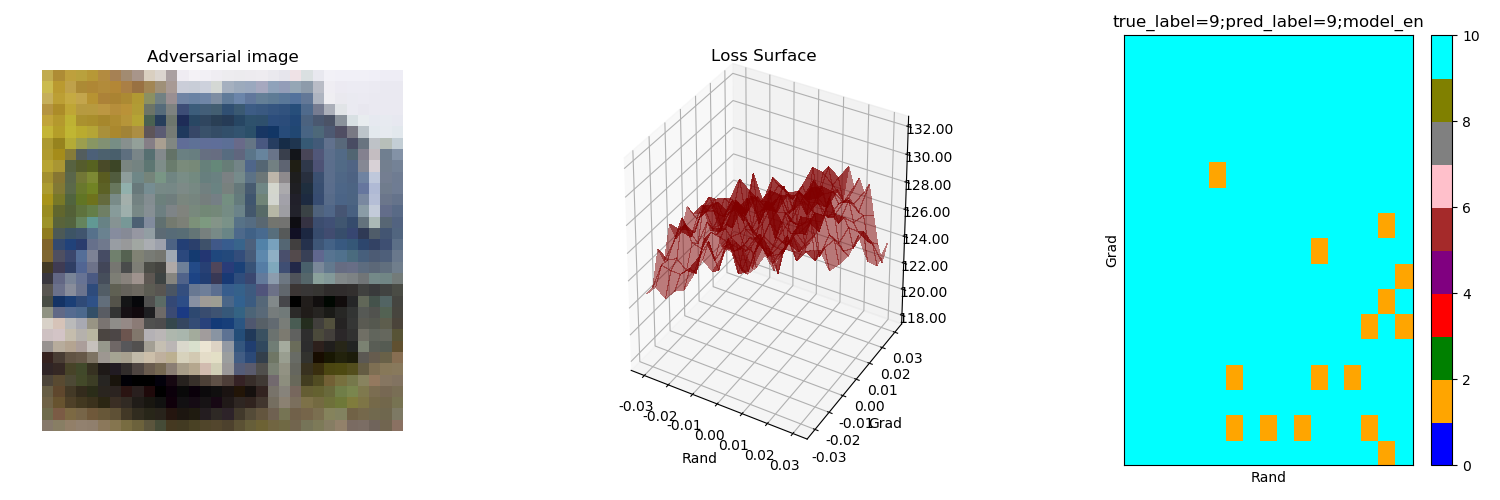}
\par\end{centering}
}
\par\end{centering}
\caption{Loss surface around adversarial example of ADP method. Left: Adversarial
input. Middle: Loss surface. Right: Predicted labels.\label{fig:surface_adp}}
\end{figure}

\begin{figure}
\begin{centering}
\subfloat[Prediction surface of model $f^{1}$.]{\begin{centering}
\includegraphics[width=1\columnwidth]{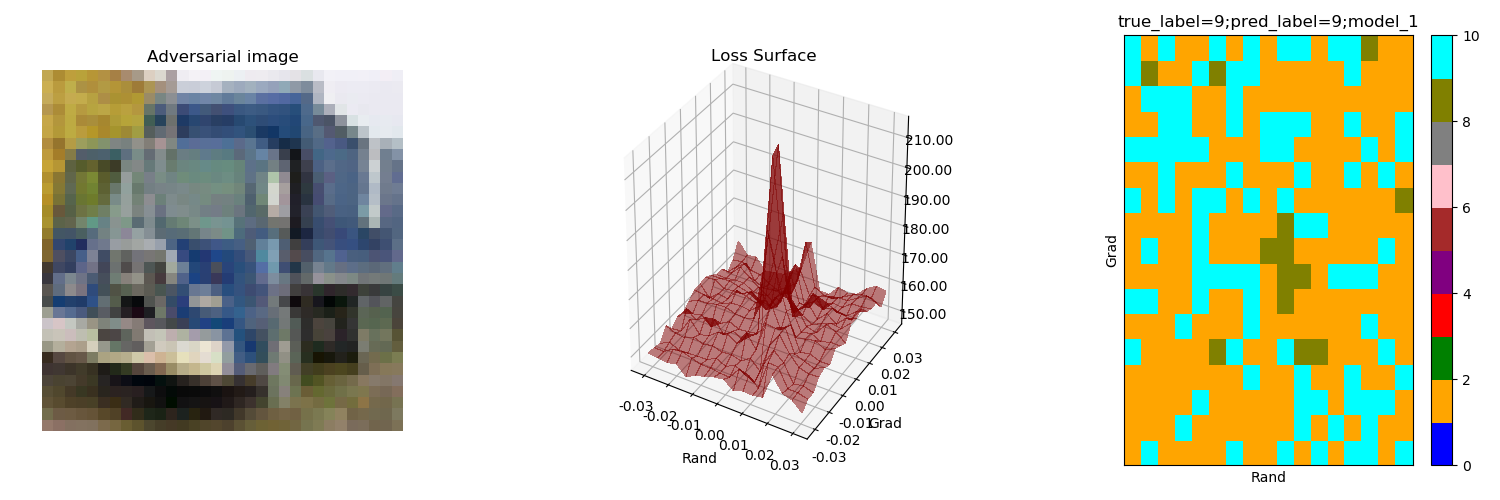}
\par\end{centering}
}\medskip{}
\subfloat[Prediction surface of model $f^{2}$.]{\begin{centering}
\includegraphics[width=1\columnwidth]{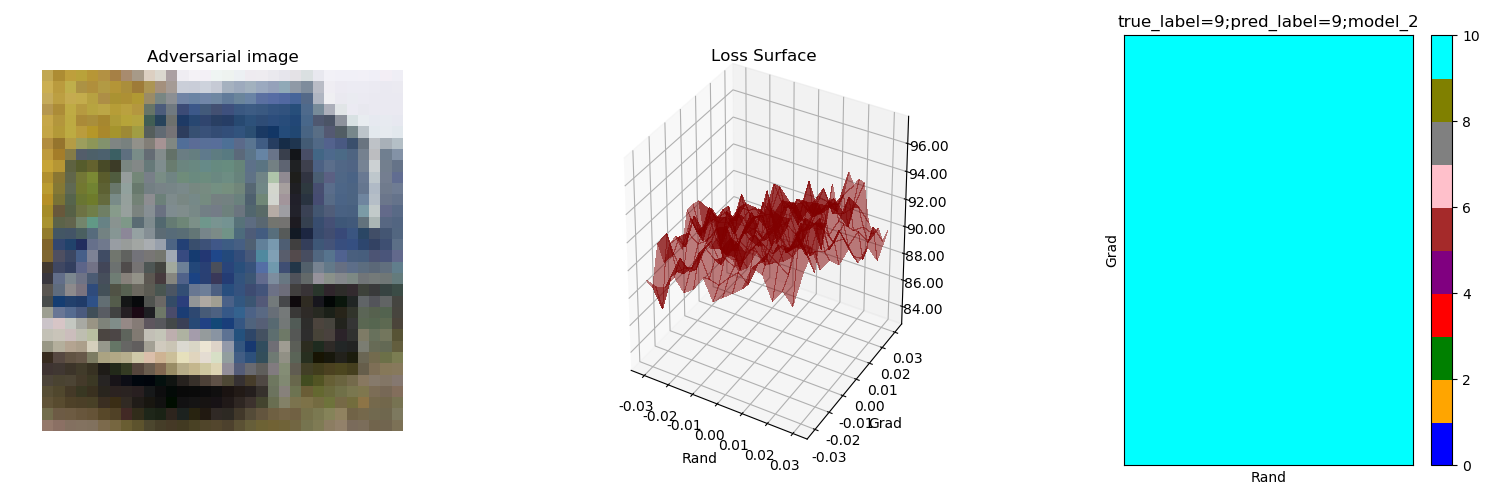}
\par\end{centering}
}\medskip{}
\subfloat[Prediction surface of model $f^{en}$.]{\begin{centering}
\includegraphics[width=1\columnwidth]{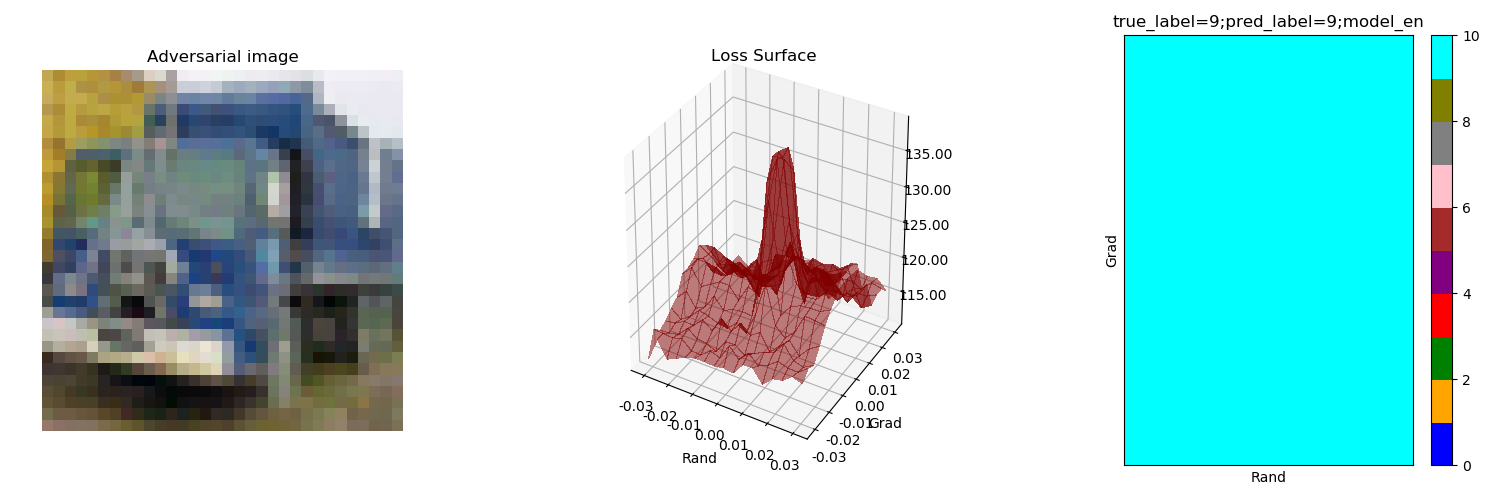}
\par\end{centering}
}
\par\end{centering}
\caption{Loss surface around adversarial example of CCE-RM method. Left: Adversarial
input. Middle: Loss surface. Right: Predicted labels.\label{fig:surface_ours}}
\end{figure}

\end{document}